\documentclass{article}

\PassOptionsToPackage{numbers, compress}{natbib}

\usepackage[preprint]{neurips_2025}

\newcommand{\sys}{\mbox{\textsc{DUMP}}\xspace}

\usepackage[utf8]{inputenc} %
\usepackage[T1]{fontenc}    %
\usepackage{url}            %
\usepackage{amsfonts}       %
\usepackage{nicefrac}       %
\usepackage{microtype}      %

\usepackage{pifont}

\usepackage{booktabs}
\usepackage{multirow}
\usepackage{graphicx}
\usepackage{float}
\usepackage{mathtools}
\usepackage{ifthen}
\usepackage{bm}
\usepackage{fp}
\usepackage{siunitx}
\usepackage{amsthm}
\usepackage{caption}

\usepackage{subcaption}
\usepackage{xspace}
\usepackage{xcolor}
\usepackage{amsmath, amsthm, amssymb}
\usepackage{pifont}

\usepackage{wrapfig}
\usepackage{pifont}

\usepackage{algorithm}
\usepackage{algorithmicx}
\usepackage[noend]{algpseudocode}
\usepackage{comment}
\algnewcommand\algorithmicinput{\textbf{Input:}}
\algnewcommand\Input{\item[\algorithmicinput]}
\algnewcommand\algorithmicoutput{\textbf{Output:}}
\algnewcommand\Output{\item[\algorithmicoutput]}
\algnewcommand{\LineComment}[1]{\State \(\triangleright\) #1}

\usepackage[hidelinks,breaklinks,colorlinks]{hyperref}
\hypersetup{
  linkcolor={blue!70!black},
  citecolor={red!70!black},
  urlcolor={blue!70!black}
}

\usepackage{hyperref}

\theoremstyle{plain}
\newtheorem{theorem}{Theorem}[section]

\theoremstyle{definition}

\theoremstyle{remark}

\usepackage{color}

\usepackage{soul}
\usepackage{xcolor}
\definecolor{DarkGreen}{RGB}{34,139,34}

\usepackage[super]{nth}

\title{\sys: Automated Distribution-Level Curriculum Learning for RL-based
LLM Post-training}

\author{%
  Zhenting Wang$^{1}$
  \And Guofeng Cui$^{1}$
  \And Yu-Jhe Li $^{2}$
  \And Kun Wan$^{2*}$
  \And Wentian Zhao$^{2*}$
  \AND \textnormal{$^{1}$Rutgers University\quad $^{2}$Adobe Inc.} \\
}

\begin{document}

\maketitle
\renewcommand{\thefootnote}{\fnsymbol{footnote}}

\begin{abstract}

Recent advances in reinforcement learning (RL)-based post-training have led to notable improvements in large language models (LLMs), particularly in enhancing their reasoning capabilities to handle complex tasks. However, most existing methods treat the training data as a unified whole, overlooking the fact that modern LLM training often involves a mixture of data from diverse distributions—varying in both source and difficulty. This heterogeneity introduces a key challenge: how to adaptively schedule training across distributions to optimize learning efficiency.
In this paper, we present a principled curriculum learning framework grounded in the notion of distribution-level learnability. Our core insight is that the magnitude of policy advantages reflects how much a model can still benefit from further training on a given distribution. Based on this, we propose a distribution-level curriculum learning framework for RL-based LLM post-training, which leverages the Upper Confidence Bound (UCB) principle to dynamically adjust sampling probabilities for different distrubutions. This approach prioritizes distributions with either high average advantage (exploitation) or low sample count (exploration), yielding an adaptive and theoretically grounded training schedule.
We instantiate our curriculum learning framework with GRPO as the underlying RL algorithm and demonstrate its effectiveness on logic reasoning datasets with multiple difficulties and sources. Our experiments show that our framework significantly improves convergence speed and final performance, highlighting the value of distribution-aware curriculum strategies in LLM post-training. Code: \url{https://github.com/ZhentingWang/DUMP}.

\end{abstract}

\section{Introduction}\label{sec:intro}

Reinforcement learning (RL)-based post-training has emerged as a powerful approach for enhancing the capabilities of large language models (LLMs), particularly in areas requiring structured reasoning, multi-step inference, and task-specific generalization~\citep{ouyang2022training, rafailov2023direct, shao2024deepseekmath, guo2025deepseek}. By leveraging reward signals derived from task performance, human feedback, or domain-specific metrics, RL provides a flexible alternative to supervised fine-tuning. Unlike imitation-based methods that merely mimic reference outputs, RL-based approaches allow models to optimize directly toward behavioral objectives, making them especially effective for boosting model performance on complex reasoning and agentic tasks.

While RL-based post-training has become a key technique for enhancing LLM capabilities in reasoning, alignment, and coding, one foundational challenge remains underexplored: \emph{how to dynamically schedule training across heterogeneous data distributions}.
In practice, LLMs are post-trained on datasets drawn from a wide variety of sources---ranging from factual QA to math problems and coding tasks---each differing in knowledge/capability relevance, and learning difficulty~\citep{longpre2023flan,lee2023rlaif,lambert2024t}. This heterogeneity is evident in large-scale post-training datasets such as T\"ulu~3~\cite{lambert2024t}, where prompts span general dialogue, logic puzzles, STEM problems, and multilingual instructions, with widely varying counts, formats, and alignment objectives.
More recently, \textit{next-generation post-training pipelines} (e.g., Seed-Thinking v1.5~\cite{seed_thinking_v15}) have shifted toward \textit{synthetic data generation with controllable parameters}---e.g., configuring logical puzzle difficulty. This allows fine-grained control over the \textit{data distribution}, making \textbf{distribution-level curriculum learning both feasible and increasingly important}.
Despite this, most RL-based pipelines still treat all data distributions equally---uniformly sampling tasks throughout training or relying on static, hand-designed curricula. This static treatment ignores the model’s evolving learning needs and underutilizes the training budget. 
Moreover, it is difficult to handcraft effective curricula when the post-training data comes from multiple distributions lacking clear difficulty labels.
As reinforcement learning becomes increasingly used in post-training and training costs continue to rise, a data-driven curriculum mechanism that dynamically prioritizes learnable distributions is not just desirable, but necessary.

This motivates the need for \textit{automated distribution-level curriculum learning}: a dynamic strategy that adjusts sampling probabilities across data distributions throughout training. While prior work has explored instance-level curricula based on sample difficulty~\citep{pattnaik2024enhancing}, and static/heuristic multi-stage schedules have been applied in LLM post-training~\citep{team2025kimi,xie2025logic}, little attention has been paid to automated, distribution-level scheduling—especially in the context of RL for capability-oriented post-training. The central challenge lies in identifying signals that reflect the current learnability of each distribution and in designing algorithms that can stably and efficiently leverage these signals to guide sampling.

In this paper, we present \sys{} (Automated \textbf{D}istribution-level c\textbf{U}rriculu\textbf{M} learning for RL-based LLM \textbf{P}ost-training), a simple but theoretically grounded approach to address this challenge. Our central insight is that the \textit{magnitude of policy advantages}—the expected absolute difference between a model’s predicted return and its baseline value—serves as a natural proxy for distribution-level learnability. High advantages on specific data distribution indicate underfitting and high potential for improvement on it, while low advantages suggest diminishing returns. Moreover, the statistical reliability of these advantage estimates improves with the number of samples drawn from each distribution.
\sys{} operationalizes this insight by using bandit-style Upper Confidence Bound (UCB) scores to schedule distribution sampling. It maintains a sliding window of recent advantage magnitudes for each distribution and computes a score that balances exploitation (high advantage) and exploration (low visitation). These scores are normalized via a softmax to form sampling weights, which are then used to generate training batches. Unlike fixed or heuristic curricula, \sys{} adapts throughout training based on empirical signals, and can be seamlessly integrated into standard LLM RL pipelines. We instantiate \sys{} with GRPO~\citep{shao2024deepseekmath}, but the method is compatible with any advantage-based RL algorithm.
We evaluate \sys{} on logic reasoning corpora. Our experiments show that \sys{} significantly accelerates convergence and yields stronger performance compared to uniform sampling. Furthermore, we provide theoretical analysis that supports the use of absolute advantages as a surrogate for distribution-level learnability, formalizing its connection to sample efficiency and regret minimization.

We summarize our contributions as follows. \ding{172} We highlight the underexplored challenge of curriculum learning at the distribution level for RL-based post-training aimed at capability enhancement. \ding{173} We propose \sys{}, a theoretically grounded framework that leverages advantage-based UCB scores to adaptively guide training over data distributions. \ding{174} We demonstrate \sys{}'s effectiveness through empirical results and theoretical analysis, showing that it enables faster, more efficient improvement on LLM capabilities.

\vspace{-0.2cm}
\section{Background}\label{sec:related}
\vspace{-0.2cm}
\noindent
\textbf{RL-based LLM Post-training.}
Reinforcement learning (RL) plays a central role in post-training large language models (LLMs), especially for tasks involving reasoning, subjective preference, or long-horizon control. The RLHF framework ~\cite{ouyang2022training,christiano2017deep,ziegler2019fine,bai2022training,glaese2022improving} laid the foundation by aligning models using reward signals derived from human preferences.
Beyond preference alignment, recent RL-based post-training approaches have notably enhanced LLMs' capabilities in complex reasoning tasks, particularly coding and mathematics. For instance, RL post-trained model OpenAI o1~\cite{o1}, o3~\cite{o3,el2025competitive}, DeepSeek-R1~\cite{guo2025deepseek} significantly outperform LLMs without RL post-training such as pre-trained versions of GPT-4o~\cite{hurst2024gpt} and DeepSeek-V3~\cite{liu2024deepseekv3} on challenging mathematics and coding benchmarks (e.g., AIME~\cite{di_zhang_2025} and Codeforces~\cite{codeforces}).
Proximal Policy Optimization (PPO)~\citep{schulman2017proximal} is widely used in post-training due to its clipped objective, which stabilizes training by preventing large policy updates. PPO remains a strong baseline in many LLM alignment settings.
Direct Preference Optimization (DPO)~\citep{rafailov2023direct} simplifies the pipeline by replacing RL rollouts with a classification-style loss derived from a KL-constrained reward maximization objective. While DPO works well on pairwise preference data, it does not naturally support group-wise or comparative feedback.
Group Relative Policy Optimization (GRPO)~\citep{shao2024deepseekmath} addresses this limitation by leveraging group-based feedback. For each input prompt \(x\), GRPO samples a group of \(G\) candidate outputs \(\{o_1, \dots, o_G\} \sim \pi_{\text{ref}}(\cdot|x)\) from a frozen reference policy \(\pi_{\text{ref}}\). Each output \(o_i\) is assigned a reward \(r_i\), and the advantage of \(o_i\) is computed by normalizing its reward relative to others in the group:
\begin{equation}
\hat{A}_i = \frac{r_i - \text{mean}(\{r_1, \dots, r_G\})}{\text{std}(\{r_1, \dots, r_G\}) + \epsilon},
\end{equation}
where \(\epsilon > 0\) is a small constant for numerical stability. These normalized advantages capture the relative quality of outputs within the group.
The model policy \(\pi_\theta\) is then updated by maximizing the following clipped surrogate objective:
{\footnotesize
\begin{equation}
\mathcal{J}_{\text{GRPO}}(\theta) = \mathbb{E}_{x, \{o_i\}} \left[ \frac{1}{G} \sum_{i=1}^G \min \left( \frac{\pi_\theta(o_i|x)}{\pi_{\text{old}}(o_i|x)} \hat{A}_i, \text{clip}\left( \frac{\pi_\theta(o_i|x)}{\pi_{\text{old}}(o_i|x)}, 1 - \epsilon, 1 + \epsilon \right) \hat{A}_i \right) - \beta \, \mathbb{D}_{\text{KL}}(\pi_\theta \| \pi_{\text{ref}}) \right],
\end{equation}
}
where \(\pi_\theta(o_i|x)\) is the probability assigned by the current model to output \(o_i\), \(\pi_{\text{old}}(o_i|x)\) is the same under the model from previous step, and \(\pi_{\text{ref}}(o_i|x)\) is that under the reference model. The first term inside the summation is a clipped policy ratio scaled by \(\hat{A}_i\), similar to PPO~\citep{schulman2017proximal}, which prevents overly large updates. The outer expectation is taken over prompts \(x\) and their sampled output groups \(\{o_i\}\). The second term is a KL divergence penalty that regularizes the updated policy \(\pi_\theta\) to stay close to \(\pi_{\text{ref}}\), weighted by a hyperparameter \(\beta\).
This formulation eliminates the need for an explicit value baseline and stabilizes training by comparing outputs within local groups.

\noindent
\textbf{Curriculum Learning for RL.}
Curriculum learning~\cite{bengio2009curriculum,graves2017automated} organizes training by progressing from easy to hard examples. In RL, curricula often follow task complexity~\cite{justesen2018illuminating,wang2019paired,li2020towards}, or are learned via teacher-student frameworks modeled as partially observable Markov decision process~\cite{matiisen2019teacher,portelas2020teacher}. With the adoption of RL in LLM post-training, curriculum learning has shown potential for improving both training efficiency and model effectiveness. For example, 
Curri-DPO~\cite{pattnaik2024enhancing} constructs instance-level curricula by ranking preference pairs based on the score gap between preferred and dispreferred responses, introducing harder pairs gradually during DPO fine-tuning. Kimi k1.5~\cite{team2025kimi} and Logic-RL~\cite{xie2025logic}, on the other hand, use manually defined heuristic curricula with fixed training stages, e.g., models are first trained on “easy” samples for a pre-specified number of steps, then switched to “hard” samples. These strategies rely on static schedules and heuristic difficulty labels, without adapting to the model’s learning progress. 
While these works demonstrate the benefit of curriculum learning in LLM post-training, most existing approaches focus on instance-level difficulty or use static, manually designed strategies. In contrast, automatic curriculum learning at the distribution level, especially in RL-based post-training, remains underexplored. In this paper, we propose \sys{} to fill this gap by adaptively scheduling training over distributions using advantage-based learnability signals.

\vspace{-0.2cm}
\section{Method}\label{sec:method}
\vspace{-0.2cm}

In this section, we introduce \sys{}, a distribution-level curriculum learning framework for RL-based LLM post-training. We first introduce expected absolute advantage as a proxy for learnability, and formalize the scheduling problem as a multi-armed bandit. We then describe a UCB-based strategy to guide distribution selection, followed by the full implementation of \sys{}.

\vspace{-0.1cm}
\subsection{Measuring Learnability via Absolute Advantage}
\vspace{-0.1cm}

We aim to dynamically assess the usefulness of different data distributions during LLM reinforcement learning post-training. Intuitively, a distribution is more useful (or ``learnable'') if the model can gain more from training on its samples. To help understand and measure the learnability of the data samples from different distributions, we provide the following theorem:

\begin{theorem}[Expected Advantage Magnitude Reflects Learnability]
\label{th:learnability}
Given a policy \(\pi_{\theta}\) and a data distribution \(d\),
the expected absolute advantage \(\mathbb{E}_{x \sim d} \left[ \mathbb{E}_{o_i \sim \pi_\theta(\cdot|x)} \left[ |\hat{A}_i| \right] \right]\) serves as a proxy for how much that distribution \(d\) can help the model improve, where the distribution $d$ consisting of prompts $x \sim d$, each prompt has a group of sampled outputs $\{o_1, \dots, o_n\}$, and \(\hat{A}_i\) denotes the advantage of output \(o_i\).
\end{theorem}

The proof can be found in \autoref{sec:proof_learnability}.
Intuitively, if training on a distribution results in a larger expected advantage magnitude, then that distribution is considered more learnable. The advantage function measures the deviation between an action’s predicted value and its actual return; a large advantage—either positive or negative—indicates that the model’s current policy is still far from optimal on those samples but has a large potential to improve. A small advantage magnitude does not necessarily imply mastery—it may also occur when a task is too difficult or noisy for the model to learn from effectively, resulting in weak or unstable learning signals. To capture this deviation in both directions, we take the absolute value of the advantage. Without this, positive and negative advantages within a batch may cancel out, masking the true extent of the model’s uncertainty or suboptimality. By averaging the absolute advantage over multiple sampled outputs and prompts, we obtain a robust estimate of how much learning signal remains in a given distribution.
This expected absolute advantage thus acts as a practical proxy for distribution-level learnability: it reflects how much the model can benefit from training on that distribution. It also has the strength of being lightweight to compute in RL pipelines, as advantage estimates are already generated during rollout.

\vspace{-0.1cm}
\subsection{Formalizing Distribution-Level Curriculum Learning as Multi-armed Bandit}
\vspace{-0.1cm}

We aim to design a curriculum learning strategy that dynamically allocates training focus across multiple data distributions to maximize overall model improvement. Let \(\mathcal{D} = \{d_1, \dots, d_N\}\) be a set of data distributions. At each training step, we sample a batch of examples \(\mathcal{B}_t\) by drawing prompts from these distributions according to a learnable sampling policy, and use the batch to update model parameters \(\theta\) via reinforcement learning. The goal is to assign higher sampling probabilities to distributions that offer greater learning potential, thereby maximizing cumulative capability gain.

As motivated in~\autoref{th:learnability}, we quantify the learning potential of a distribution \(d\) via its expected absolute advantage, defined as \(L(d) = \mathbb{E}_{x \sim d} \left[ \mathbb{E}_{o \sim \pi_\theta(\cdot|x)} \left[ |\hat{A}(o)| \right] \right]\). Our objective is to dynamically adjust the sampling distribution over \(\mathcal{D}\) such that, over the training horizon \(T\), we approximately maximize the total expected learnability gain \(\sum_{t=1}^T \mathbb{E}_{d \sim P_t} [L(d)]\), where \(P_t\) is the sampling distribution at step \(t\).
This setup resembles a multi-armed bandit (MAB) problem, where each distribution acts as an arm and its reward corresponds to its learnability. 
In this setting, the central challenge is to estimate and balance each distribution’s potential: exploiting those with high observed advantage while still exploring under-sampled ones that may offer long-term benefit. To this end, we adopt the classic Upper Confidence Bound (UCB) principle~\citep{auer2002finite}, which provides theoretical guarantees for balancing exploration and exploitation in bandit problems. Specifically, UCB-based algorithms achieve sublinear regret compared to the optimal fixed-arm strategy, and we show in~\autoref{sec:theory_ucb} that applying UCB on empirical advantage statistics yields a near-optimal schedule under mild assumptions.
To allow smoother allocation of sampling probabilities without hard cutoffs and reducing variance in learning, 
we adopt a soft-selection mechanism: instead of choosing one distribution at each step, we compute a UCB score for every distribution and normalize the scores with a softmax function to obtain a sampling distribution.
This soft-selection formulation preserves the spirit of UCB—higher scoring distributions are sampled more—but enables partial exploration of all arms, and it is easier to integrate into LLM training pipelines.
The resulting sampling distribution provides a convex mixture over data sources, where each distribution \(d_j\) is selected with probability.
Each training batch is then composed by drawing examples from multiple distributions in proportion to their scores.
To estimate learnability in practice, we maintain a sliding window \(\mathcal{A}_{d_j}^w\) of recent absolute advantages for each distribution \(d_j\), and define its empirical reward as the mean absolute advantage: \(\hat{L}(d_j) = \frac{1}{|\mathcal{A}_{d_j}^w|} \sum_{a \in \mathcal{A}_{d_j}^w} |a|\). We also track the total number of samples drawn from each distribution \(n_{d_j}\), and the global sample count \(n_{\text{total}} = \sum_j n_{d_j}\).
The UCB score for each distribution is:
\begin{equation}\label{eq:ucb}
\text{UCB}(d_j) = \hat{L}(d_j) + \sqrt{\frac{2 \log (n_{\text{total}} + 1)}{n_{d_j} + 1}}
\end{equation}
The first term encourages exploitation of distributions with high observed advantages, while the second term ensures sufficient exploration of rarely sampled distributions.
To obtain the final sampling weights, we apply a softmax over the UCB scores. Specifically, the probability of selecting distribution \(d_j\) is computed as:
\(P(d_j) = \frac{\exp(\text{UCB}(d_j) / \tau)}{\sum_{j=1}^N \exp(\text{UCB}(d_j) / \tau)}\),
where \(\tau > 0\) is a temperature hyperparameter that controls the sharpness of the sampling distribution. A lower \(\tau\) results in more peaked selection around the top-scoring distributions, while a higher \(\tau\) leads to a smoother, more exploratory curriculum.
This bandit-based formulation provides a lightweight, adaptive, and reward-sensitive curriculum learning mechanism. It balances the need to focus on learnable distributions while avoiding premature neglect of underexplored ones. In the next section, we present the complete algorithmic implementation of \sys{}, including its integration with rollout procedures and online statistics tracking.

\begin{algorithm}[t]
\caption{Automated Distribution-Level Curriculum Learning with UCB Sampling}
\label{alg:cl}
{\bf Input:} Dataset \(\mathcal{D} = \{d_1, \dots, d_N\}\); pre-trained model parameters \(\theta\)\\
{\bf Output:} Post-trained model parameters \(\theta\)
\begin{algorithmic}[1]
\Function{\sys}{$\mathcal{D}$, $\theta$}
    \State \(\triangleright\) Initialize distribution-level statistics
    \For{each \(d_j \in \mathcal{D}\)}
        \State \(\mathcal{A}_{d_j}^w \gets \text{[ ]}\) \Comment{Sliding window for absolute advantages}
        \State \(n_{d_j} \gets 0\) \Comment{Total samples seen from \(d_j\)}
        \State \(P(d_j) \gets \frac{1}{N}\) \Comment{Equal initial weights}
    \EndFor

    \For{training step \(t = 1, 2, \dots, T\)}
        \State Sample batch \(\mathcal{B}_t\) from \(\mathcal{D}\) according to \(P(d_j)\)
        \State Compute advantages \(\hat{A}(o)\) for all \(o \in \mathcal{B}_t\) via model rollout
        
        \For{each \(d_j\) with samples in \(\mathcal{B}_t\)}
            \State \(n_{d_j} \gets n_{d_j} + |\mathcal{B}_{t,d_j}|\) \Comment{Update sample count; \(\mathcal{B}_{t,d_j}\): subset of batch from \(d_j\)}

            \State \(\mathcal{A}_{d_j}^w \gets \mathcal{A}_{d_j}^w \cup \left\{ |\hat{A}(o)| \mid x \in \mathcal{B}_{t,d_j},\, o \sim \pi_\theta(\cdot|x) \right\}\) \Comment{Append new advantages from \(d_j\)}
            \State $\mathcal{A}_{d_j}^w \gets \mathcal{A}_{d_j}^w[-k:]$ \Comment{$k$: Window Size; Keep last $k$ elements}
        \EndFor
        
        \State \(\triangleright\) Compute UCB scores for each distribution
        \State \(n_{\text{total}} \gets \sum_{d_j \in \mathcal{D}} n_{d_j}\)
        \For{each \(d_j \in \mathcal{D}\)}
            \State \(\hat{L}(d_j) \gets \frac{1}{|\mathcal{A}_{d_j}^w|} \sum_{a \in \mathcal{A}_{d_j}^w} a\) \Comment{Mean of absolute advantages}
            \State \(\text{UCB}(d_j) \gets \hat{L}(d_j) + \sqrt{\frac{2 \log(n_{\text{total}} + 1)}{n_{d_j} + 1}}\) \Comment{\autoref{eq:ucb}}
        \EndFor
        
        \State \(\triangleright\) Update sampling distribution
        \State \(P(d_j) \gets \frac{\exp(\text{UCB}(d_j)/\tau)}{\sum_{j=1}^{N} \exp(\text{UCB}(d_j)/\tau)} \quad \forall d_j \in \mathcal{D}\) \Comment{\(\tau\): temperature}
        
        \State Update \(\theta\) using \(\mathcal{B}_t\) with an RL algorithm (e.g., GRPO)
    \EndFor
    \State \Return \(\theta\)
\EndFunction
\end{algorithmic}
\end{algorithm}

\vspace{-0.2cm}
\subsection{Algorithm}\label{sec:algorithm}
\vspace{-0.2cm}

The detailed curriculum learning procedure is illustrated in \autoref{alg:cl}. The algorithm takes as input a dataset \(\mathcal{D} = \{d_1, \dots, d_N\}\) composed of multiple distributions and returns the optimized model parameters \(\theta\) through a reinforcement learning loop.
In lines 3–6, we initialize per-distribution statistics: each distribution \(d_j \in \mathcal{D}\) is associated with an empty sliding window \(\mathcal{A}_{d_j}^w\) to store recent absolute advantages, a counter \(n_{d_j}\) for tracking the number of samples drawn from \(d_j\), and an initial sampling probability \(P(d_j) = \frac{1}{N}\) indicating uniform sampling.
At each training step \(t\) (line 8), a batch \(\mathcal{B}_t\) is sampled according to the current distribution weights \(P(d_j)\). Advantages \(\hat{A}(o)\) are then computed via model rollouts for each sampled output \(o \in \mathcal{B}_t\) (line 9). For every distribution \(d_j\) that contributes samples in the current batch, we update its sample count \(n_{d_j}\) (line 11), append the corresponding advantages to its sliding window \(\mathcal{A}_{d_j}^w\) (line 12), and truncate the window to retain only the most recent \(k\) entries (300 by default) in line 13. 
This ensures that our estimate of per-distribution learnability remains up-to-date and robust to noise.
In lines 15–18, we compute the Upper Confidence Bound (UCB) score \(\text{UCB}(d_j)\) for each distribution. The score consists of two terms: the empirical mean absolute advantage \(\hat{L}(d_j)\) over the sliding window \(\mathcal{A}_{d_j}^w\), and an exploration bonus inversely proportional to the square root of the number of samples \(n_{d_j}\). This balances prioritization of distributions that are either highly learnable or underexplored.
In line 20, the sampling probabilities \(P(d_j)\) are updated by applying a softmax over the UCB scores with a temperature parameter \(\tau\) (0.1 by default). 
Lower values of \(\tau\) result in sharper distributions that concentrate more heavily on top-ranked distributions, while higher \(\tau\) values induce a smoother, more exploratory curriculum.
Finally, in line 21, the model parameters \(\theta\) are updated using the current batch \(\mathcal{B}_t\) with a reinforcement learning algorithm such as GRPO. After \(T\) steps, the algorithm returns the post-trained model \(\theta\), which has been adaptively guided to learn from the most informative distributions.

\vspace{-0.2cm}
\section{Experiments and Results}
\label{sec:eval}
\vspace{-0.2cm}

In this section, we first introduce our experiments setup including used models datasets and more implementation details. We then demonstrate the results for the effectiveness of our method \sys. 
More discussion about the comparison to static heuristic curriculum~\cite{xie2025logic,team2025kimi} can be found in \autoref{sec:comparison}.

\vspace{-0.1cm}
\subsection{Experiments Setup}
\label{sec:setup}
\vspace{-0.1cm}

\noindent
\textbf{RL Algorithm and LLM Models.}
We use GRPO~\cite{shao2024deepseekmath} as the underlying RL algorithm in our experiments, which is commonly used in capability-oriented LLM post-training~\cite{guo2025deepseek}.
We use Qwen2.5-7B-Instruct-1M~\cite{yang2024qwen2} and Qwen2.5-3B-Instruct~\cite{yang2024qwen2} in our experiments.

\noindent
\textbf{Datasets and Settings.} 
Multiple datasets are used in our experiments, including
Knights and Knaves (K\&K) puzzle dataset~\cite{xie2024memorization}, RuleTaker~\cite{clark2020transformers}, ProofWriter~\cite{tafjord2020proofwriter},
AR-LSAT~\cite{zhong2021ar}, LogiQA~\cite{liu2020logiqa}, LogicNLI~\cite{tian2021diagnosing}, LongICLBench~\cite{li2024long}, GSM-8K~\cite{cobbe2021gsm8k}, and AIME 1983-2024~\cite{di_zhang_2025}. 
In our experiments, we consider three different settings. The prompt template used in shown in \autoref{fig:prompt_example} in the Appendix.

\emph{Setting 1: Post-training on K\&K puzzles with varying character numbers.}
The Knights and Knaves (K\&K) dataset~\cite{xie2024memorization} contains procedurally generated logic puzzles where each character is either a knight (always truthful) or a knave (always lying), and the goal is to infer each character’s identity. The dataset supports fine-grained difficulty control by adjusting the number of characters. We generate puzzles with 3 to 14 characters, treating each character count as a separate distribution—yielding \emph{12 distinct distributions}. Each distribution includes 900 training and 100 test samples. We post-train Qwen2.5-7B-Instruct-1M on the combined dataset across all distributions.

\emph{Setting 2: Post-training on diverse logic reasoning distributions.}
We perform post-training using a mixture of logic reasoning datasets, including RuleTaker~\cite{clark2020transformers}, ProofWriter~\cite{tafjord2020proofwriter}, AR-LSAT~\cite{zhong2021ar}, LogiQA~\cite{liu2020logiqa}, LogicNLI~\cite{tian2021diagnosing}, LongICLBench Geomotion~\cite{li2024long}, and Knights and Knaves (K\&K)~\cite{xie2024memorization}. For RuleTaker, ProofWriter, and K\&K, we further partition the data distributions by complexity levels: RuleTaker by 2, 3, and 5 required reasoning steps; ProofWriter by 3, 4, and 5 required reasoning steps; and K\&K by the number of characters (3–7). In total, we construct 15 logic distributions, each containing 400 training samples. We use Qwen2.5-7B-Instruct-1M for this setting.

\emph{Setting 3: Post-training on diverse math reasoning distributions.}
We also explore post-training on diverse math data. For AIME, we split the data into four distributions based on competition years—1983–1993, 1994–2004, 2005–2015, and 2016–2024—since problem styles evolve significantly over time. We also include GSM-8K as a complementary math dataset. This results in five math distributions in total, with 7473 (GSM-8K), 124, 194, 283, and 238 training samples, respectively. We use Qwen2.5-3B-Instruct for this setting.

\noindent
\textbf{Reward Implementation.} We adopt the rule-based reward mechanism~\citet{shao2024deepseekmath} to provide stable and hack-resistant training signals during RL-based post-training and follow the detailed reward implemetation in Logic-RL~\cite{xie2025logic}. Specifically, each model response is expected to follow a structured format with the reasoning process enclosed in \texttt{<think>} tags and the final answer enclosed in \texttt{<answer>} tags. The reward system consists of two components:

\begin{itemize}
    \item \emph{Format Reward.} A binary reward based on whether the output strictly adheres to the expected format. If the model includes exactly one well-formed \texttt{<think>} and one \texttt{<answer>} section in the correct order, it receives a reward of \(+1\); otherwise, it receives a penalty of \(-1\).
    
    \item \emph{Answer Reward.} We evaluate the correctness of the final answer. If the predicted identities fully match the ground truth, the model receives a reward of \(+2\); if the answer is incorrect, \(-1.5\); and if the answer is missing or unparsable, \(-2\).
\end{itemize}

\begin{table*}[htbp]
\centering
\scriptsize
\setlength\tabcolsep{2pt}
\begin{minipage}[t]{0.52\textwidth}
\centering
\begin{tabular}{@{}ccccc@{}}
\toprule
Data Distribution      &  & without DUMP &  & with DUMP \\ \midrule
RuleTaker 2 Steps      &  & 0.79         &  & \textbf{0.79}      \\
RuleTaker 3 Steps      &  & 0.76         &  & \textbf{1.02}      \\
RuleTaker 5 Steps      &  & 0.56         &  & \textbf{0.98}      \\
ProofWriter 3 Steps    &  & \textbf{1.18}         &  & 1.09      \\
ProofWriter 4 Steps    &  & 0.97         &  & \textbf{1.09}      \\
ProofWriter 5 Steps    &  & \textbf{1.24}         &  & 1.05      \\
AR-LSAT                &  & -0.70        &  & \textbf{-0.52}     \\
LogiQA                 &  & \textbf{1.94}         &  & 1.70      \\
LogicNLI               &  & -0.29        &  & \textbf{-0.23}     \\
LongICLBench Geomotion &  & \textbf{0.54}         &  & 0.25      \\
K \& K 3 Characters    &  & 2.00         &  & \textbf{2.00}      \\
K \& K 4 Characters    &  & 1.54         &  & \textbf{1.76}      \\
K \& K 5 Characters    &  & 1.53         &  & \textbf{1.84}      \\
K \& K 6 Characters    &  & 0.83         &  & \textbf{1.42}      \\
K \& K 7 Characters    &  & 0.56         &  & \textbf{1.02}      \\ \midrule
Average                &  & 0.90         &  & \textbf{1.17}      \\ \bottomrule
\end{tabular}
\caption{Test Answer Reward (see \autoref{sec:setup}) on diverse logic reasoning distributions (Setting 2). The model used here is Qwen2.5-7B-Instruct-1M.}
\label{tab:combined_logic}
\end{minipage}
\hfill
\begin{minipage}[t]{0.45\textwidth}
\centering
\begin{tabular}{@{}ccccc@{}}
\toprule
Data Distribution &  & without DUMP &  & with DUMP \\ \midrule
GSM-8K            &  & \textbf{1.50}         &  & 1.47      \\
AIME 1983-1993    &  & -0.76        &  & \textbf{-0.39}     \\
AIME 1994-2004    &  & -1.50        &  & \textbf{-1.02}     \\
AIME 2005-2015    &  & -0.94        &  & \textbf{-0.94}     \\
AIME 2016-2024    &  & -1.27        &  & \textbf{-1.27}     \\ \midrule
Average           &  & -0.59        &  & \textbf{-0.43}     \\ \bottomrule
\end{tabular}
\caption{Test Answer Reward (see \autoref{sec:setup}) on diverse math reasoning distributions (Setting 3). The model used here is Qwen2.5-3B-Instruct.}
\label{tab:math}
\end{minipage}
\vspace{-0.6cm}
\end{table*}

\noindent
\textbf{Other Implementation Details.} All experiments are conducted on servers equipped with 8 Nvidia A100 GPUs. Our method is implemented with VeRL~\cite{sheng2024hybridflow} LLM Reinforcement Learning framework. We use GRPO~\cite{shao2024deepseekmath} as the training algorithm and follow standard practice for actor rollout and optimization. The actor learning rate is set to \(1e{-6}\), training batch size is set to 128, and the PPO mini-batch size is 32. KL divergence regularization is applied to encourage alignment with the reference policy, with a KL loss coefficient of 0.001. Each rollout batch contains 16 responses. 
If not specified, we allow for a maximum response length of 20480 and 4096 tokens during training for Qwen2.5-7B-Instruct-1M and Qwen2.5-3B-Instruct, respectively.
The window size \(k\) and the temperature \(\tau\) in our curriculum learning framework is set to 300 and 0.1, respectively.

\begin{figure}[t]
    \begin{subfigure}[t]{0.33\columnwidth}
        \centering
        \includegraphics[width=\columnwidth]{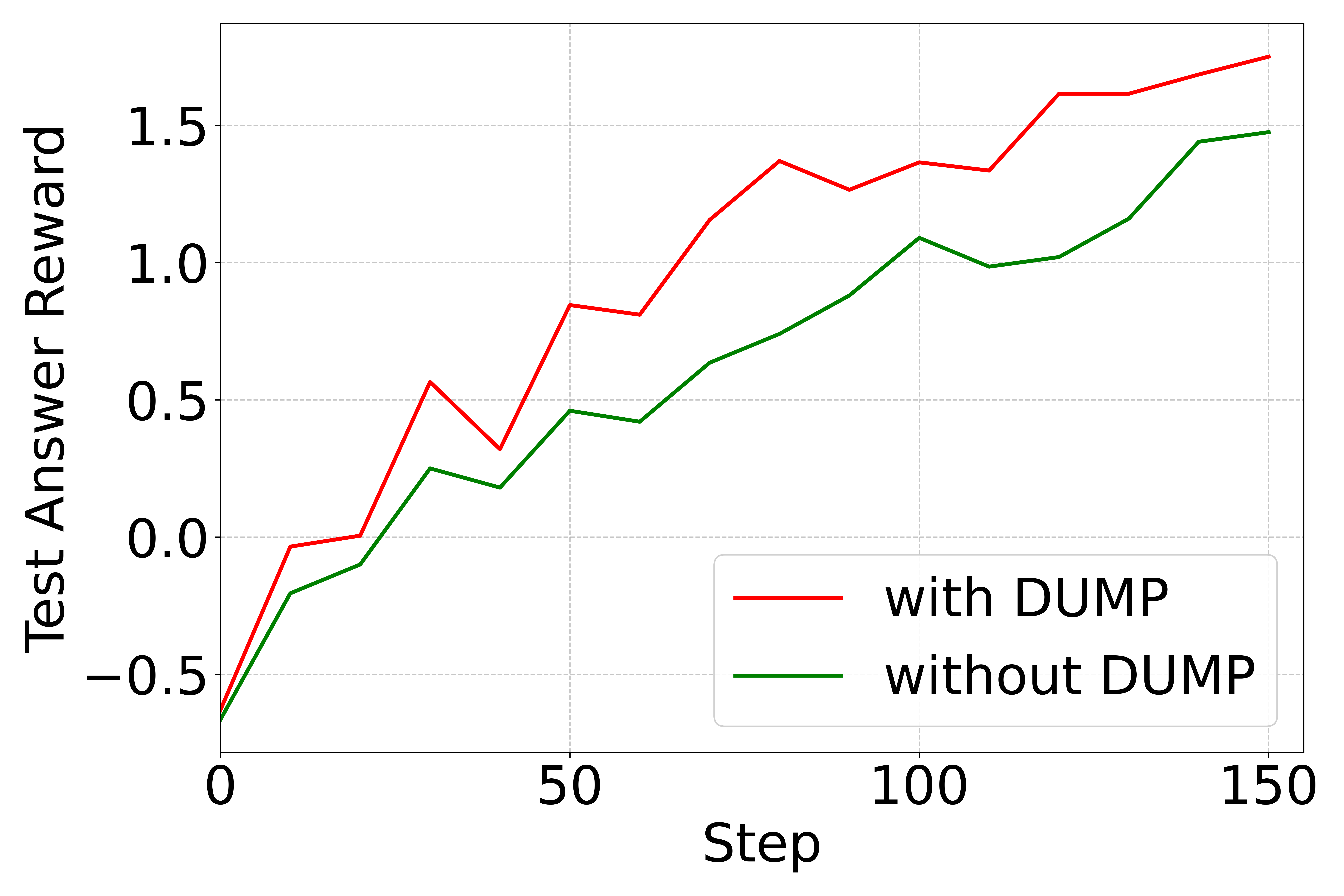}
        \caption{3 Characters}
    \end{subfigure}
    \begin{subfigure}[t]{0.33\columnwidth}
        \centering
        \includegraphics[width=\columnwidth]{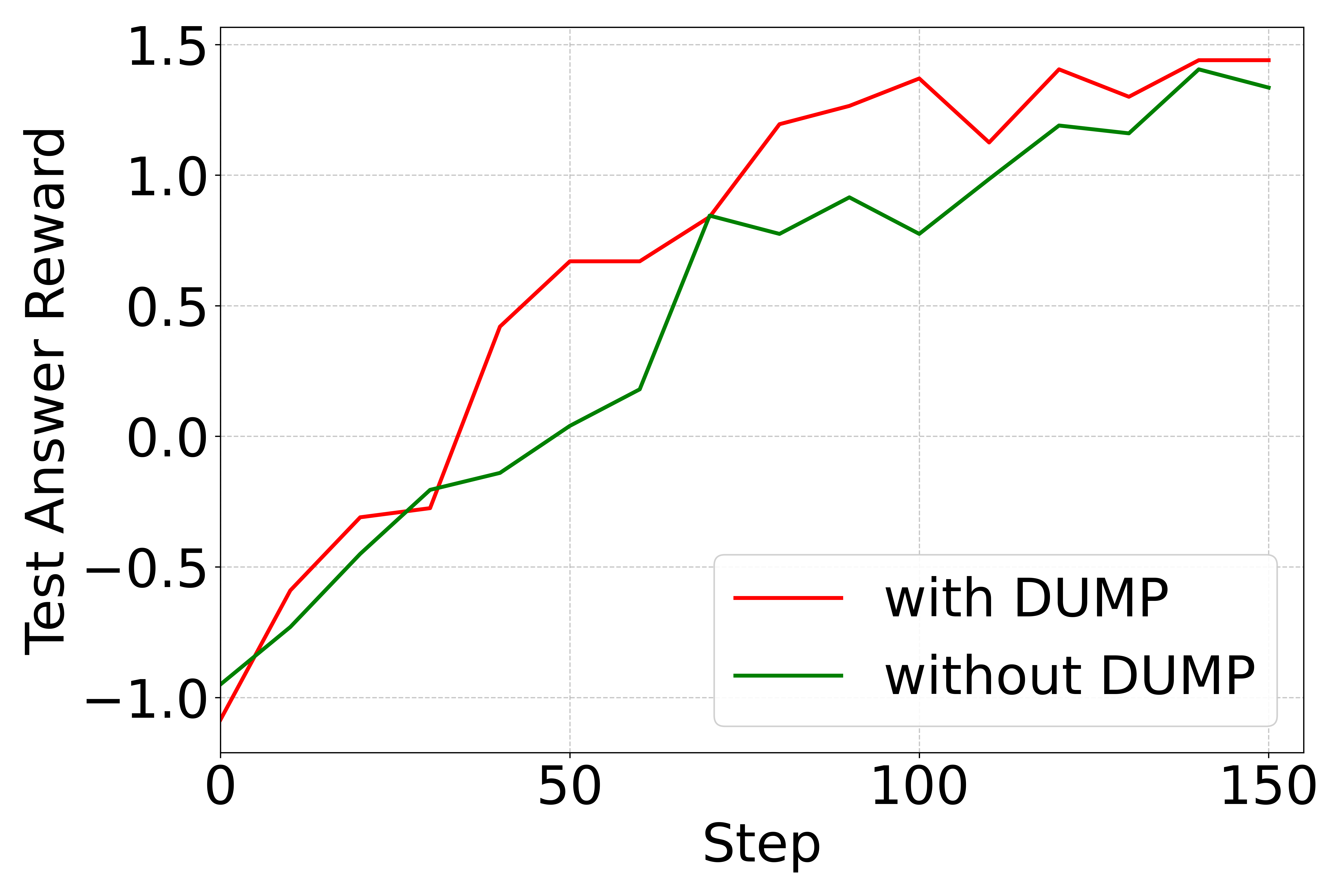}
        \caption{4 Characters}
    \end{subfigure}
    \begin{subfigure}[t]{0.33\columnwidth}
        \centering
        \includegraphics[width=\columnwidth]{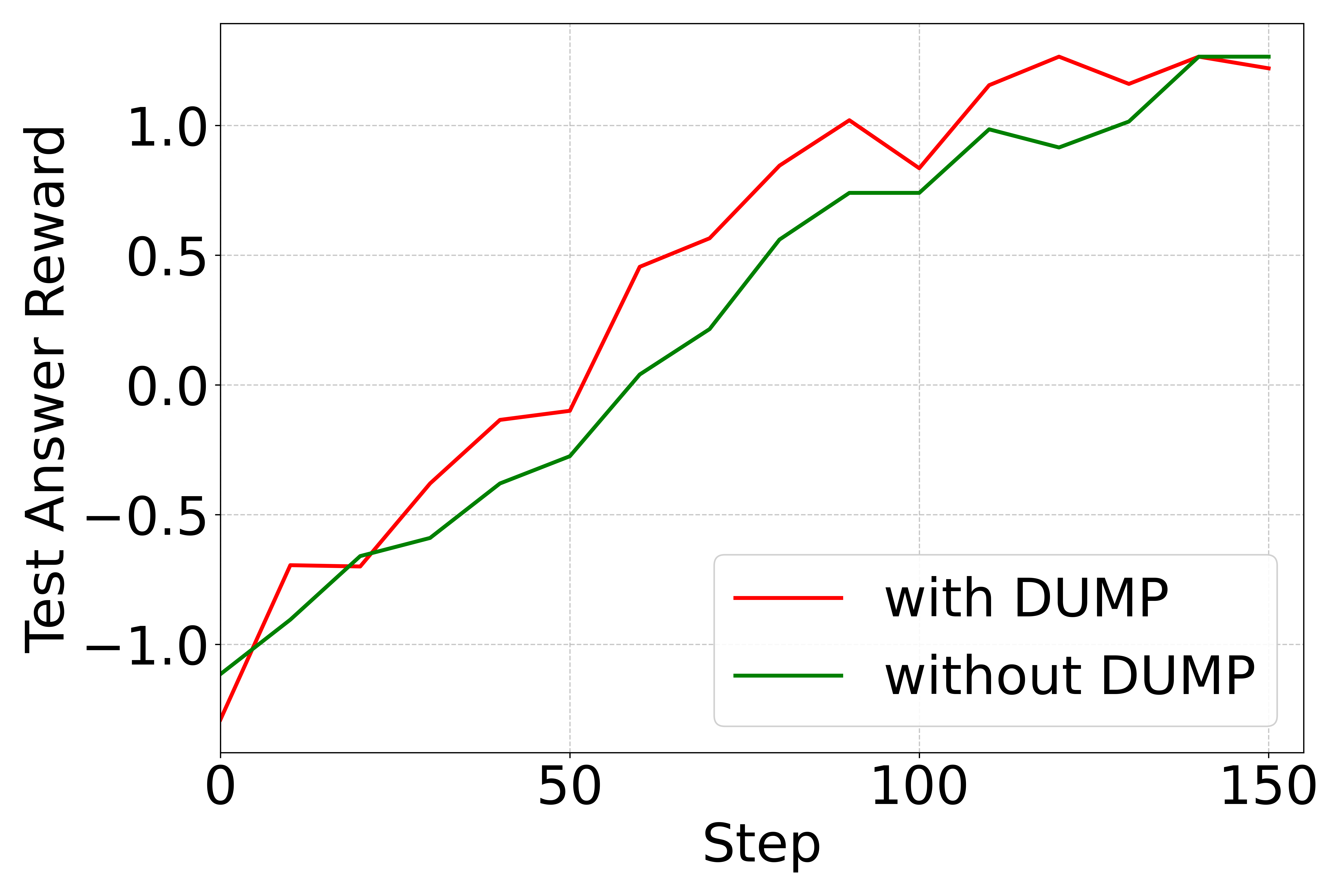}
        \caption{5 Characters}
    \end{subfigure}
    
    \begin{subfigure}[t]{0.33\columnwidth}
        \centering
        \includegraphics[width=\columnwidth]{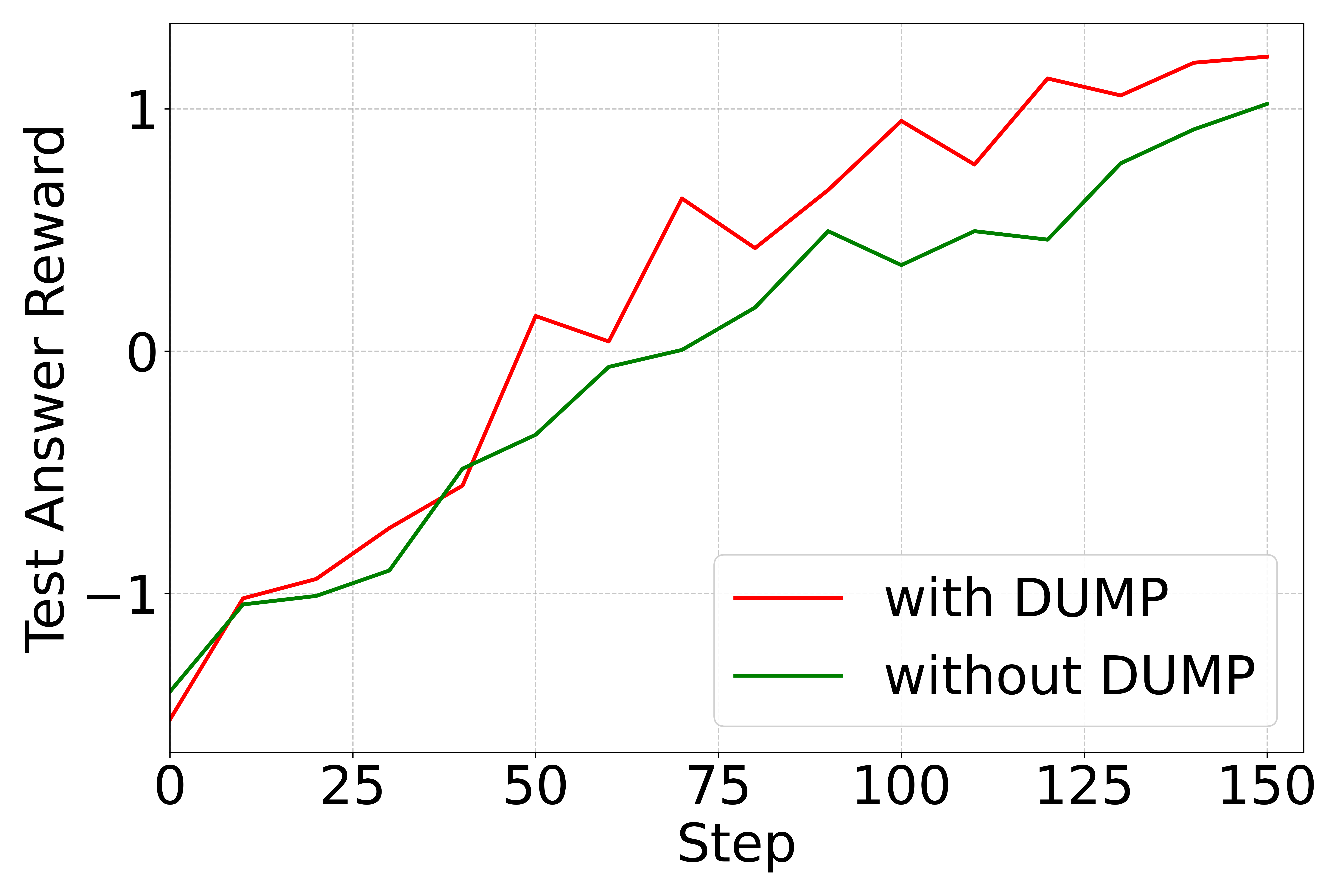}
        \caption{6 Characters}
    \end{subfigure}
    \begin{subfigure}[t]{0.33\columnwidth}
        \centering
        \includegraphics[width=\columnwidth]{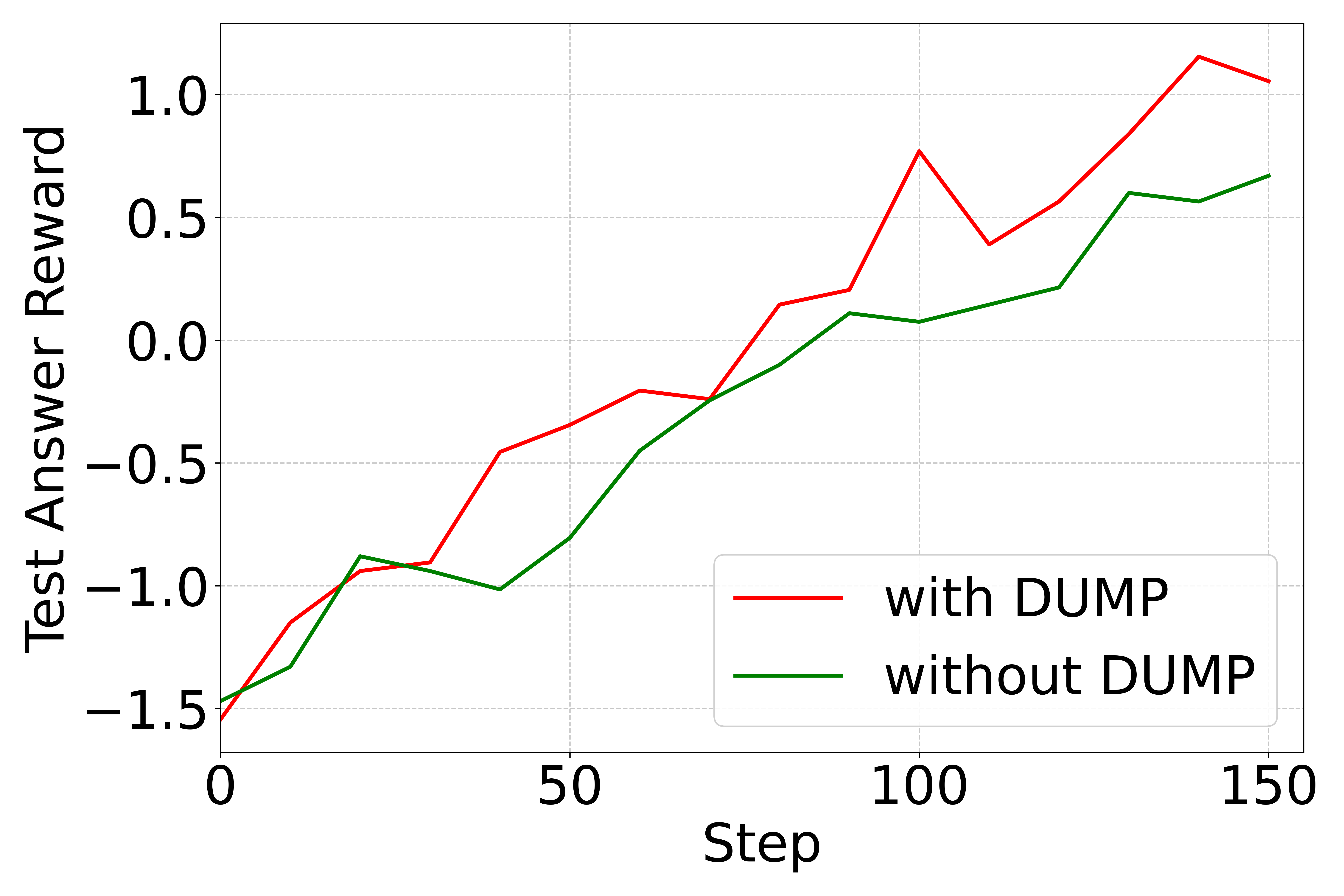}
        \caption{7 Characters}
    \end{subfigure}
    \begin{subfigure}[t]{0.33\columnwidth}
        \centering
        \includegraphics[width=\columnwidth]{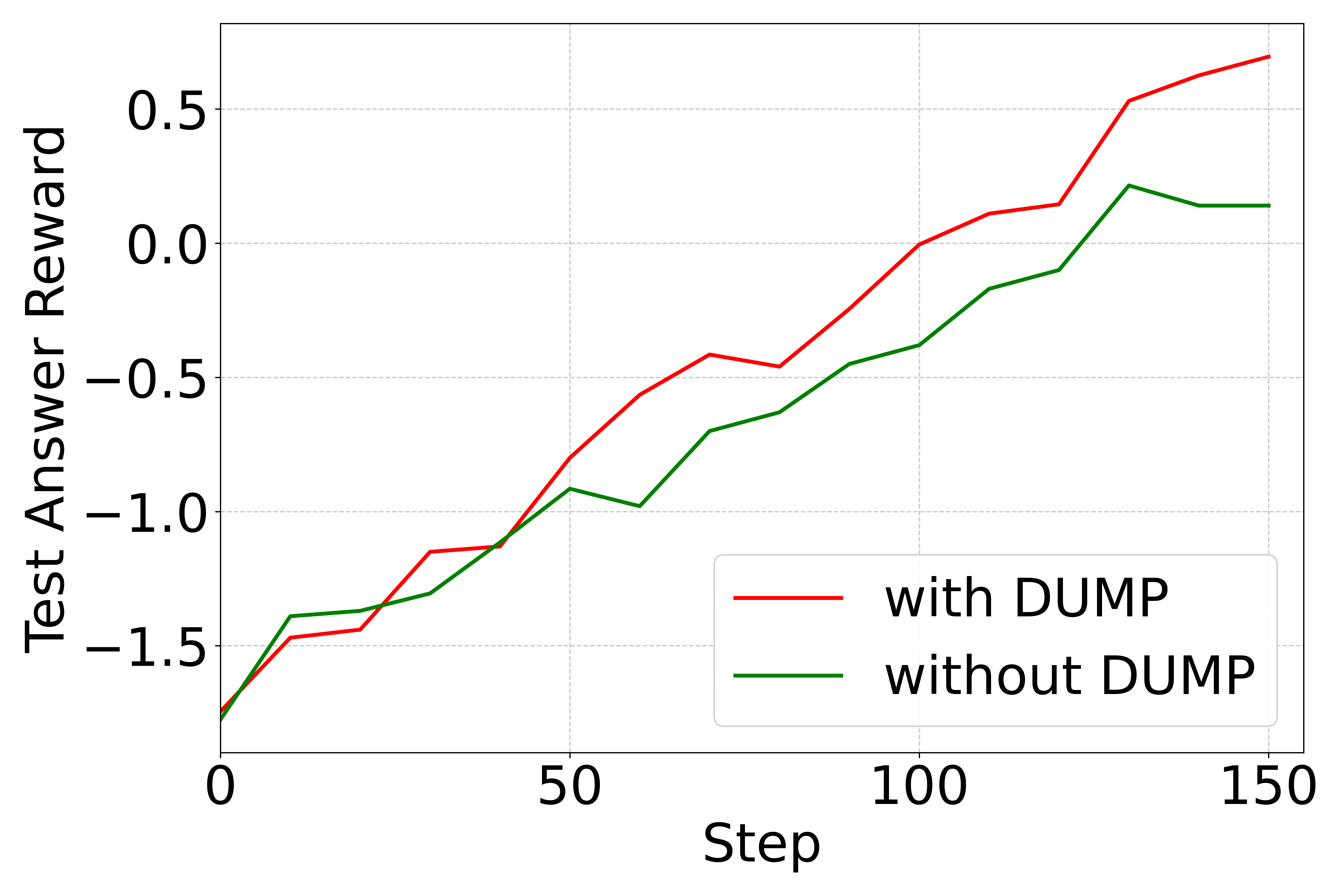}
        \caption{8 Characters}
    \end{subfigure}
    
    \begin{subfigure}[t]{0.33\columnwidth}
        \centering
        \includegraphics[width=\columnwidth]{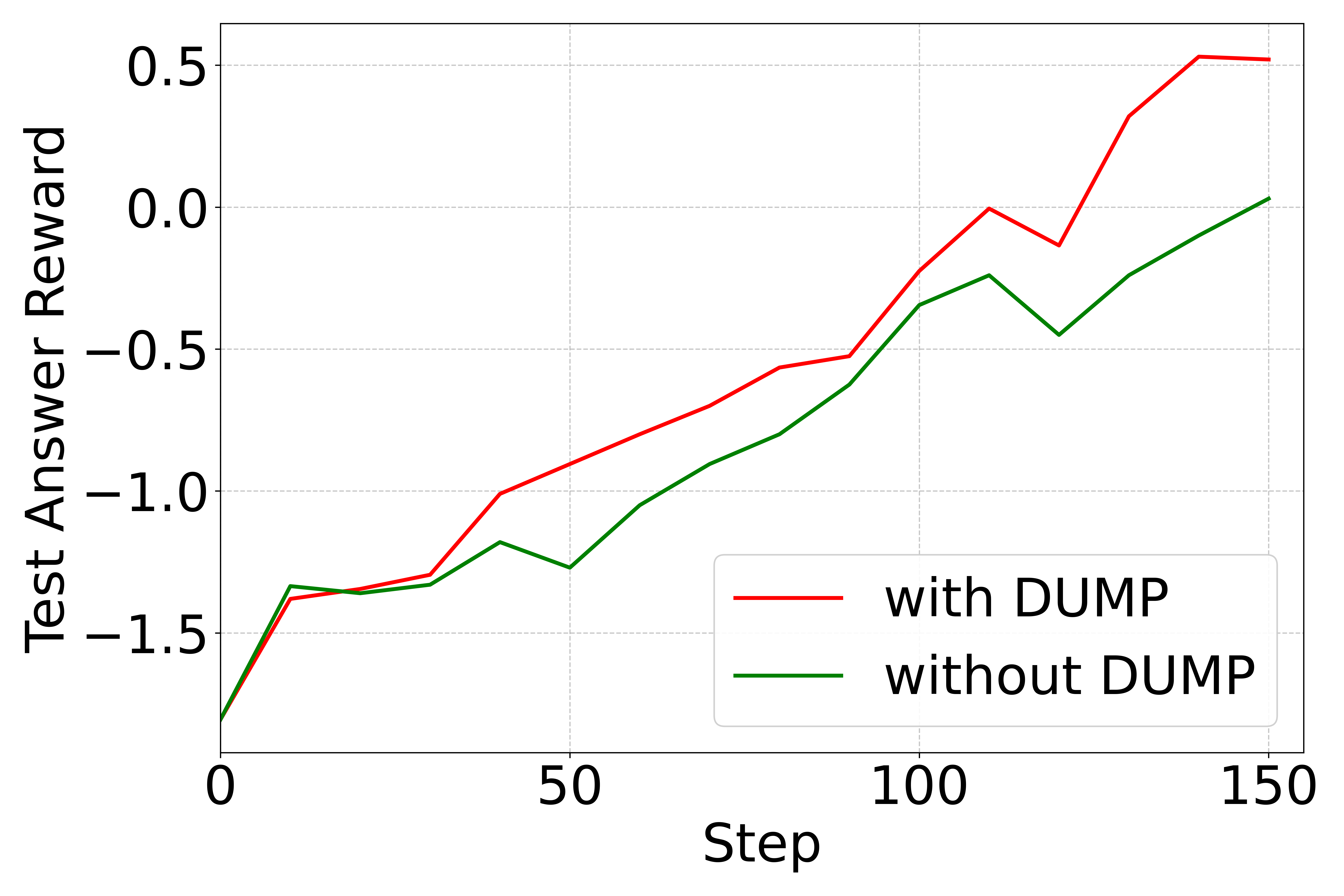}
        \caption{9 Characters}
    \end{subfigure}
    \begin{subfigure}[t]{0.33\columnwidth}
        \centering
        \includegraphics[width=\columnwidth]{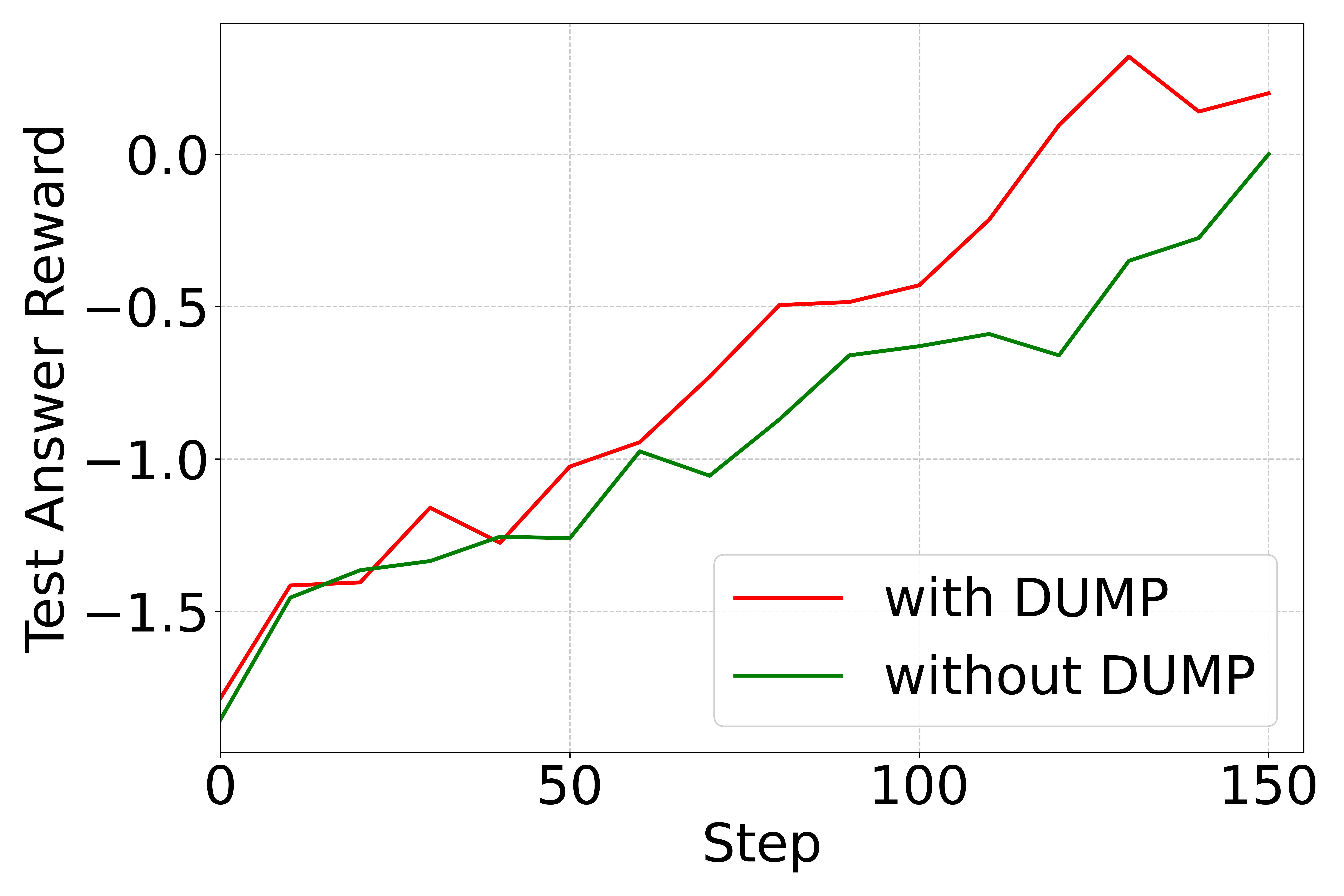}
        \caption{10 Characters}
    \end{subfigure}
    \begin{subfigure}[t]{0.33\columnwidth}
        \centering
        \includegraphics[width=\columnwidth]{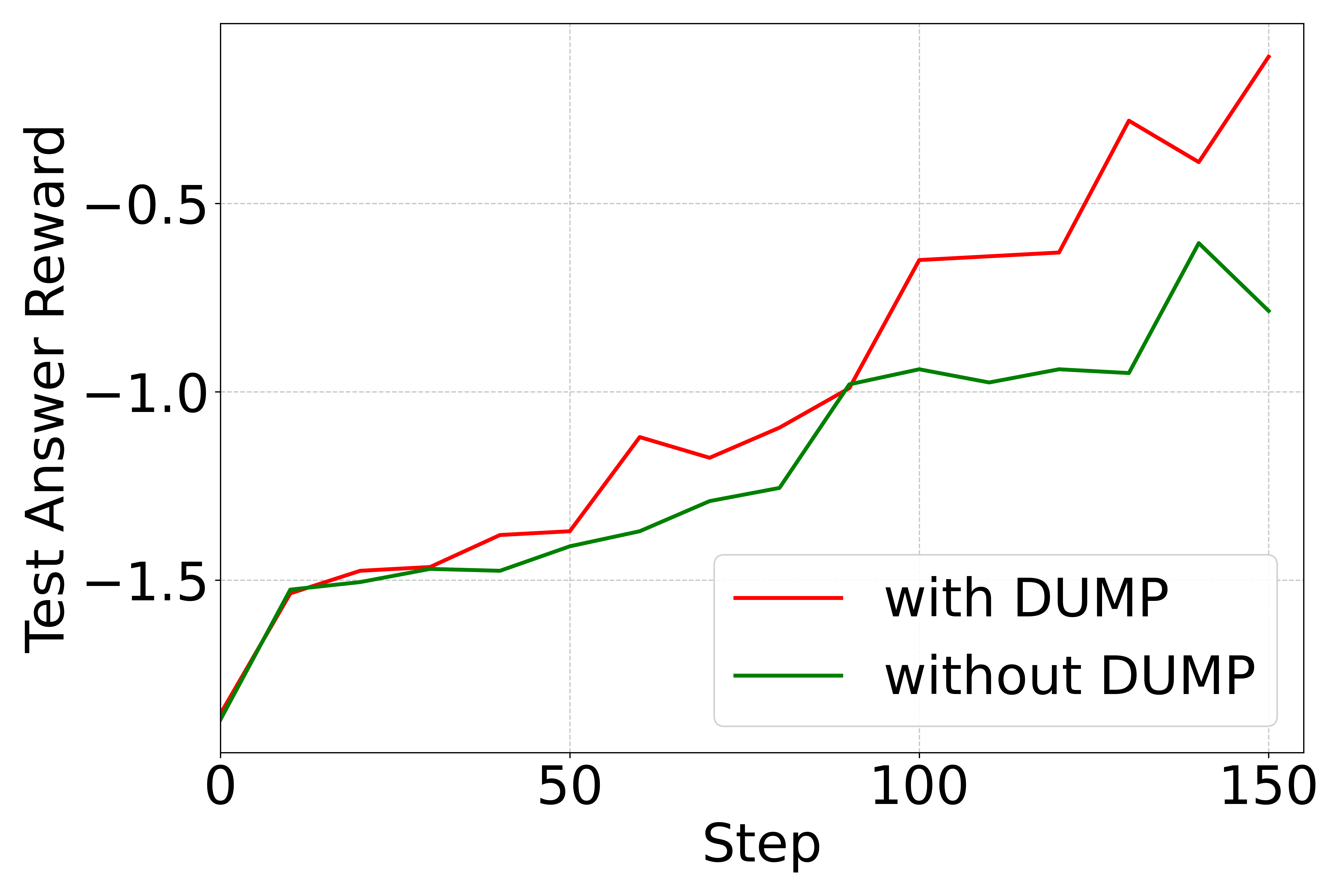}
        \caption{11 Characters}
    \end{subfigure}
    
    \begin{subfigure}[t]{0.33\columnwidth}
        \centering
        \includegraphics[width=\columnwidth]{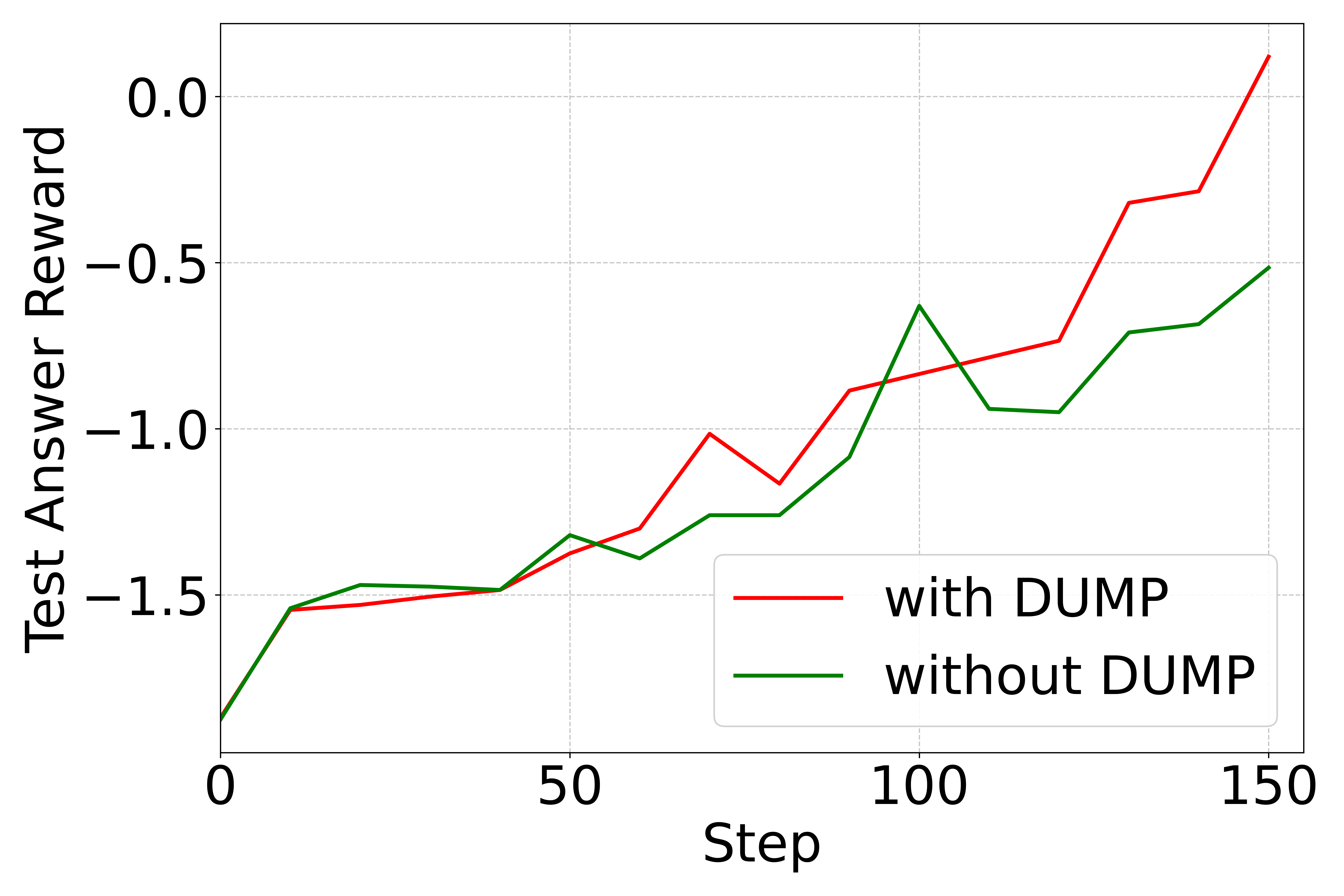}
        \caption{12 Characters}
    \end{subfigure}
    \begin{subfigure}[t]{0.33\columnwidth}
        \centering
        \includegraphics[width=\columnwidth]{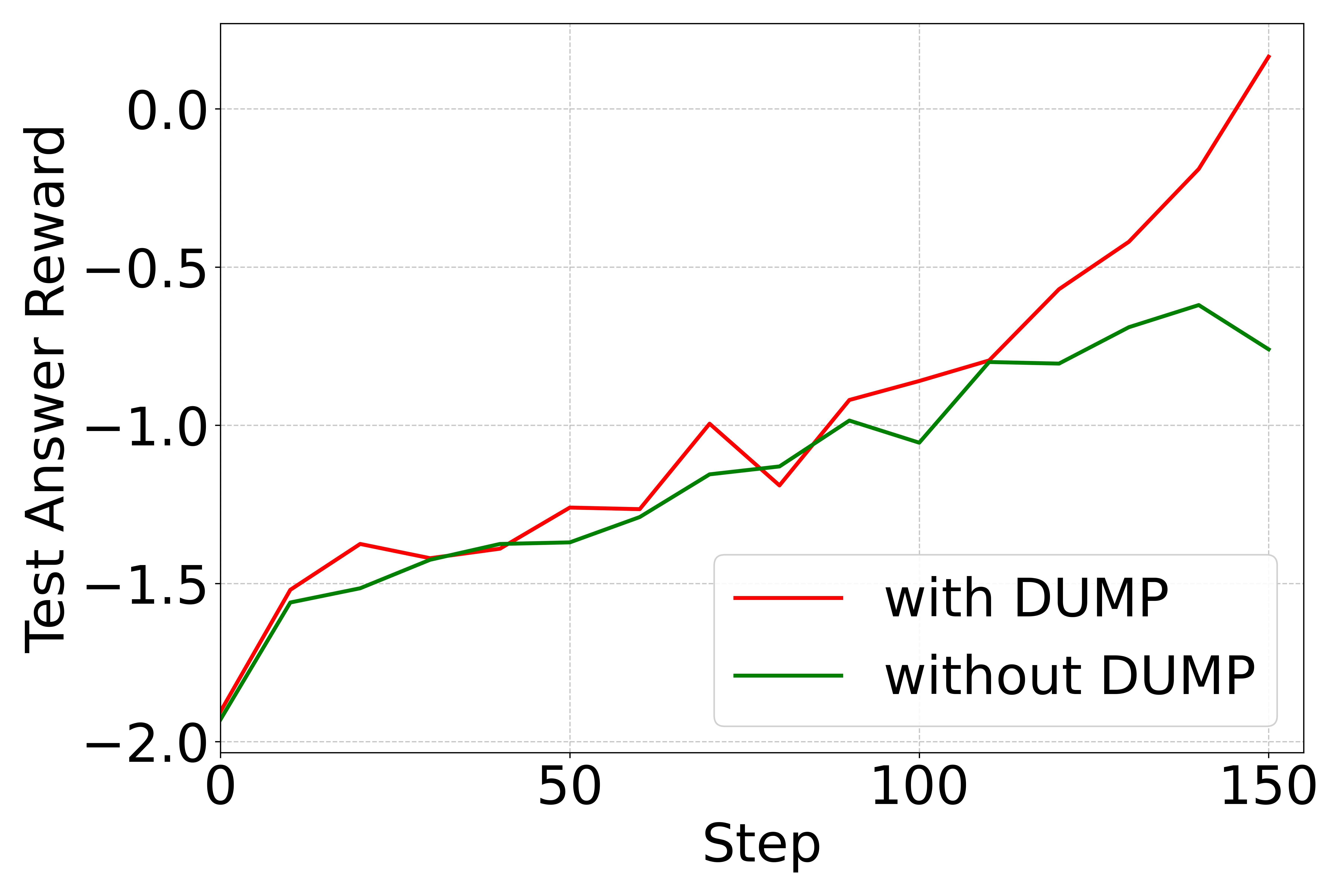}
        \caption{13 Characters}
    \end{subfigure}
    \begin{subfigure}[t]{0.33\columnwidth}
        \centering
        \includegraphics[width=\columnwidth]{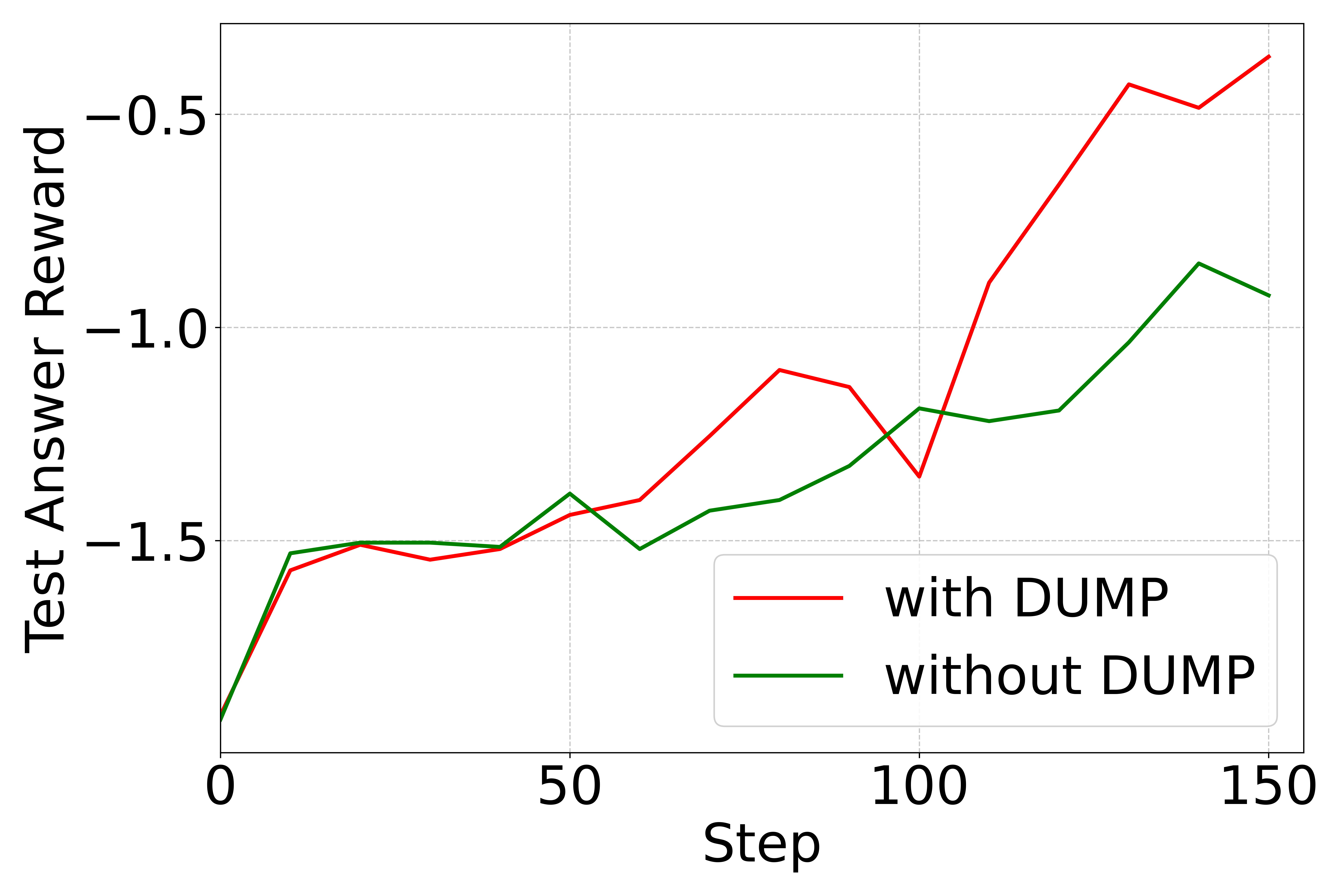}
        \caption{14 Characters}
    \end{subfigure}
    
    \caption{Effectiveness of \sys{} on the K\&K puzzle dataset mixed with 12 distributions defined by the number of characters in each puzzle (Setting 1). \sys consistently achieves higher answer reward on test dataset compared to baseline. The model used here is Qwen2.5-7B-Instruct-1M.}
    \label{fig:answer_reward}
    \vspace{-0.4cm}
\end{figure}

\begin{figure}[t]
    \begin{subfigure}[t]{0.33\columnwidth}
        \centering
        \includegraphics[width=\columnwidth]{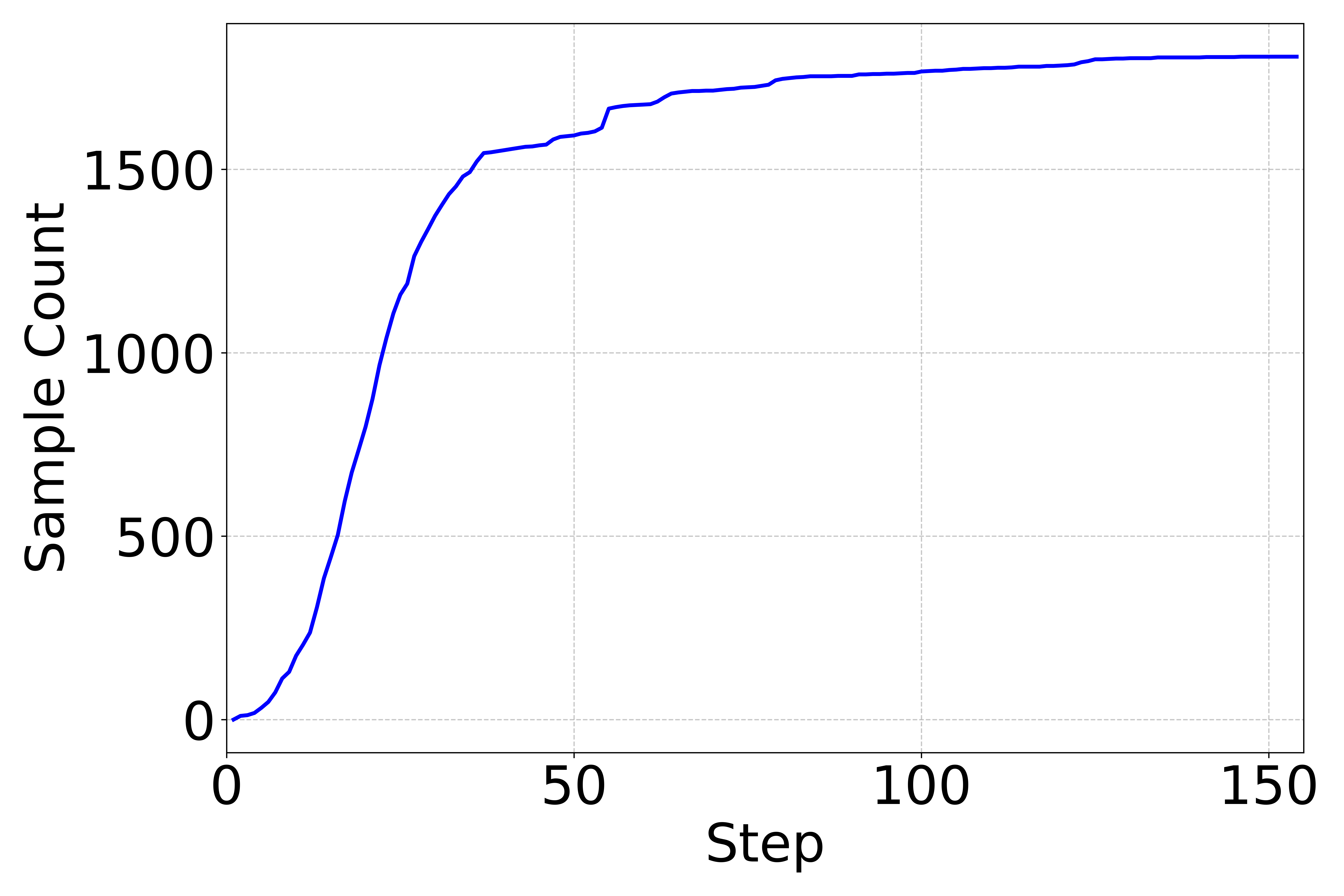}
        \caption{3 Characters}
    \end{subfigure}
    \begin{subfigure}[t]{0.33\columnwidth}
        \centering
        \includegraphics[width=\columnwidth]{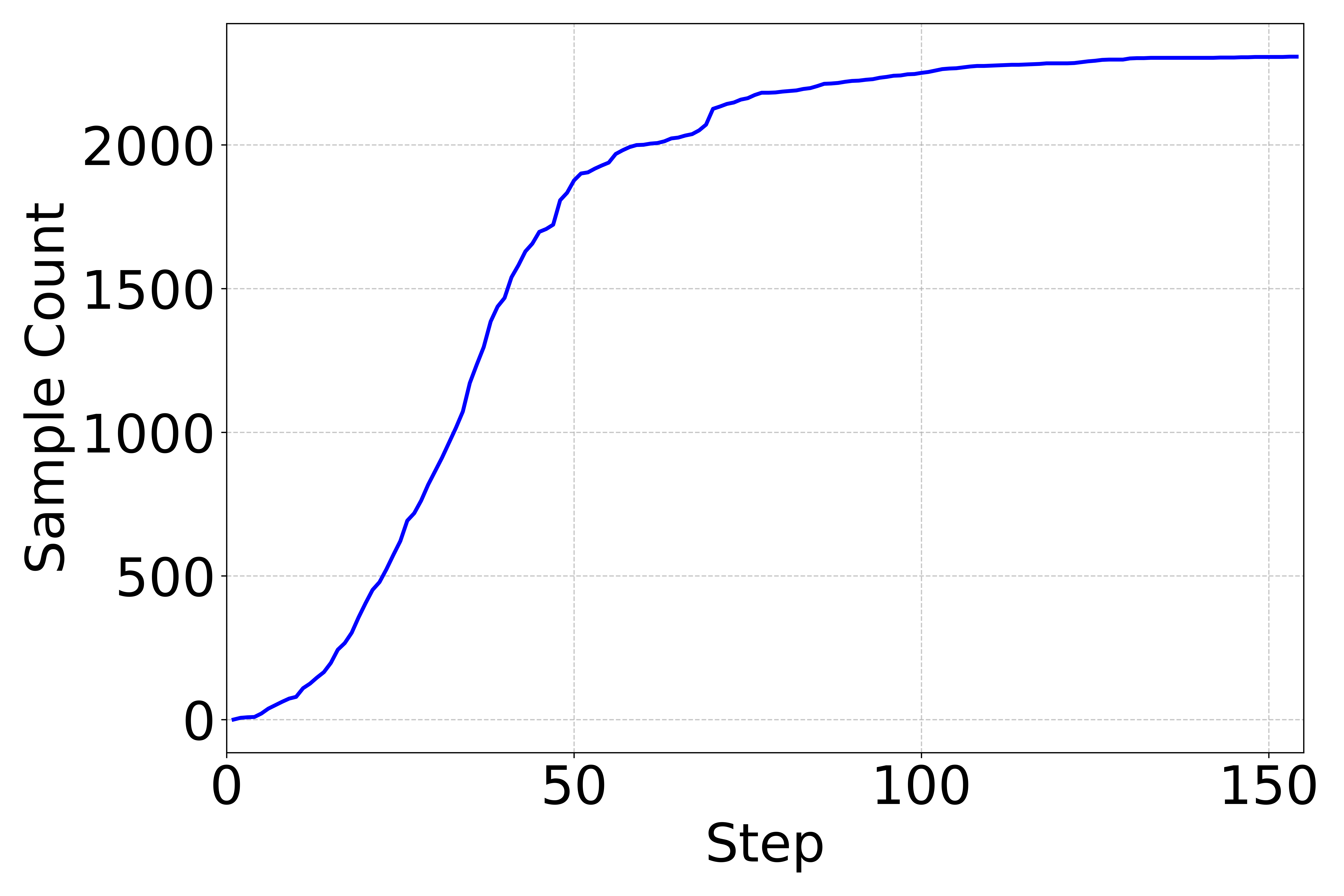}
        \caption{4 Characters}
    \end{subfigure}
    \begin{subfigure}[t]{0.33\columnwidth}
        \centering
        \includegraphics[width=\columnwidth]{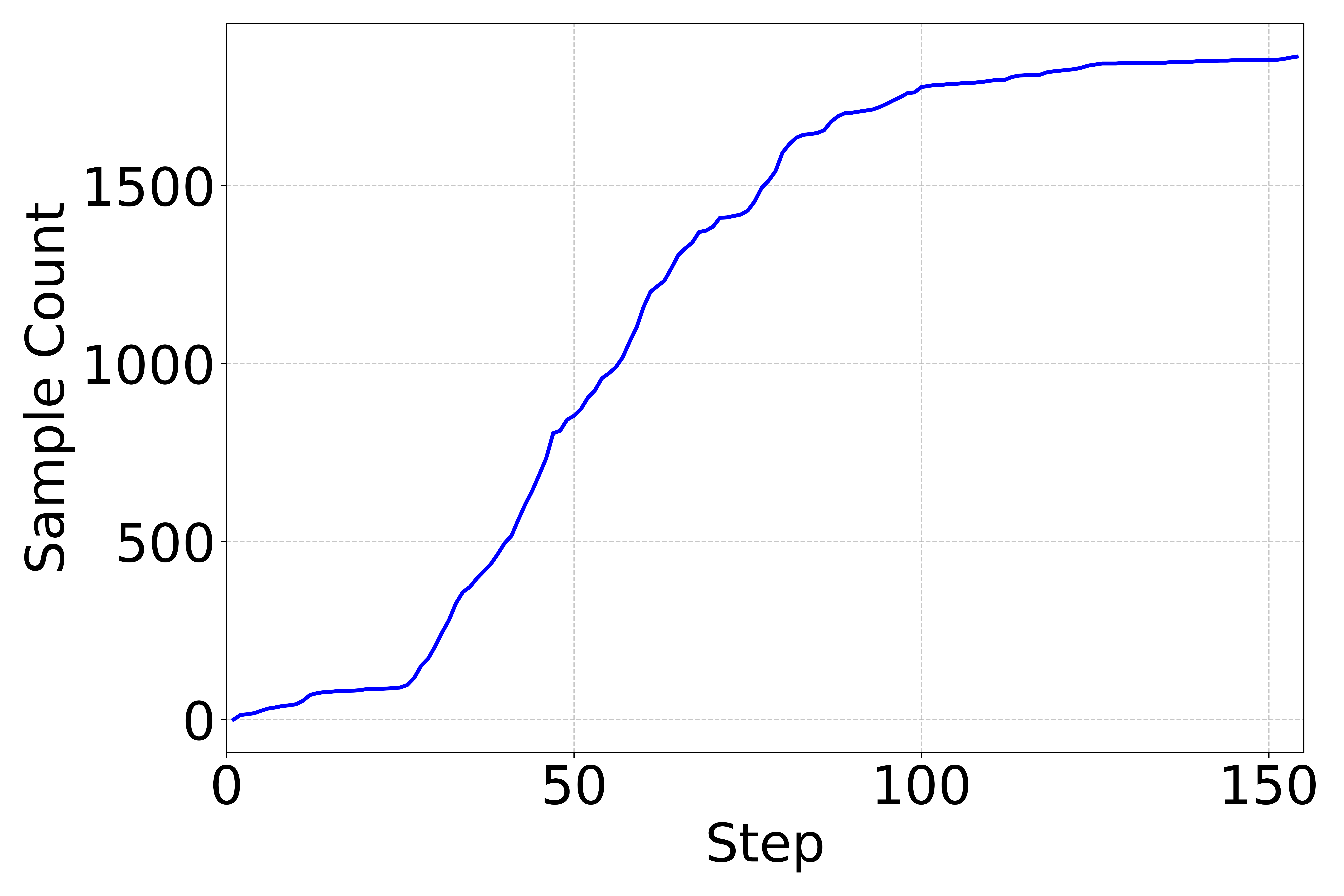}
        \caption{5 Characters}
    \end{subfigure}
    
    \begin{subfigure}[t]{0.33\columnwidth}
        \centering
        \includegraphics[width=\columnwidth]{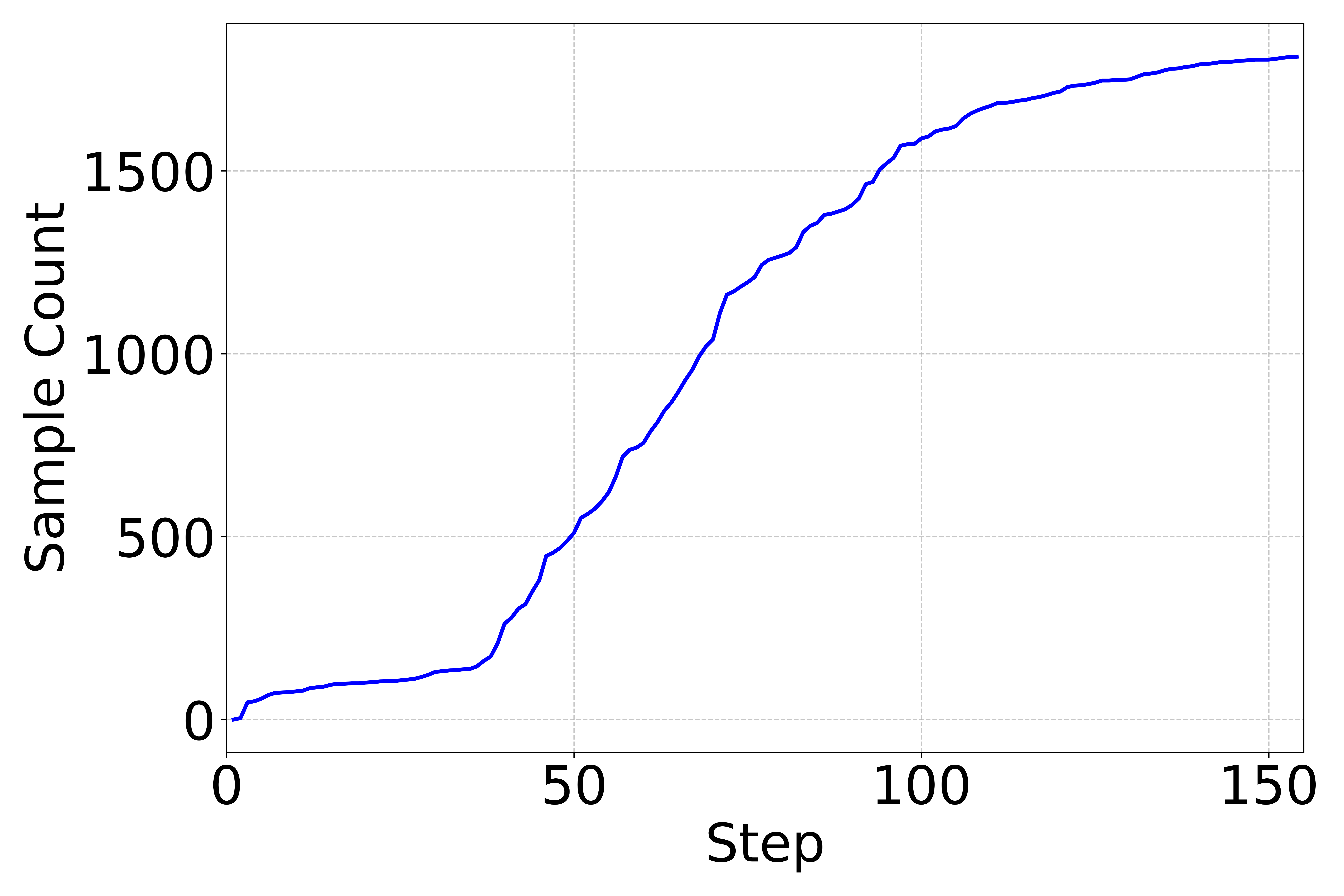}
        \caption{6 Characters}
    \end{subfigure}
    \begin{subfigure}[t]{0.33\columnwidth}
        \centering
        \includegraphics[width=\columnwidth]{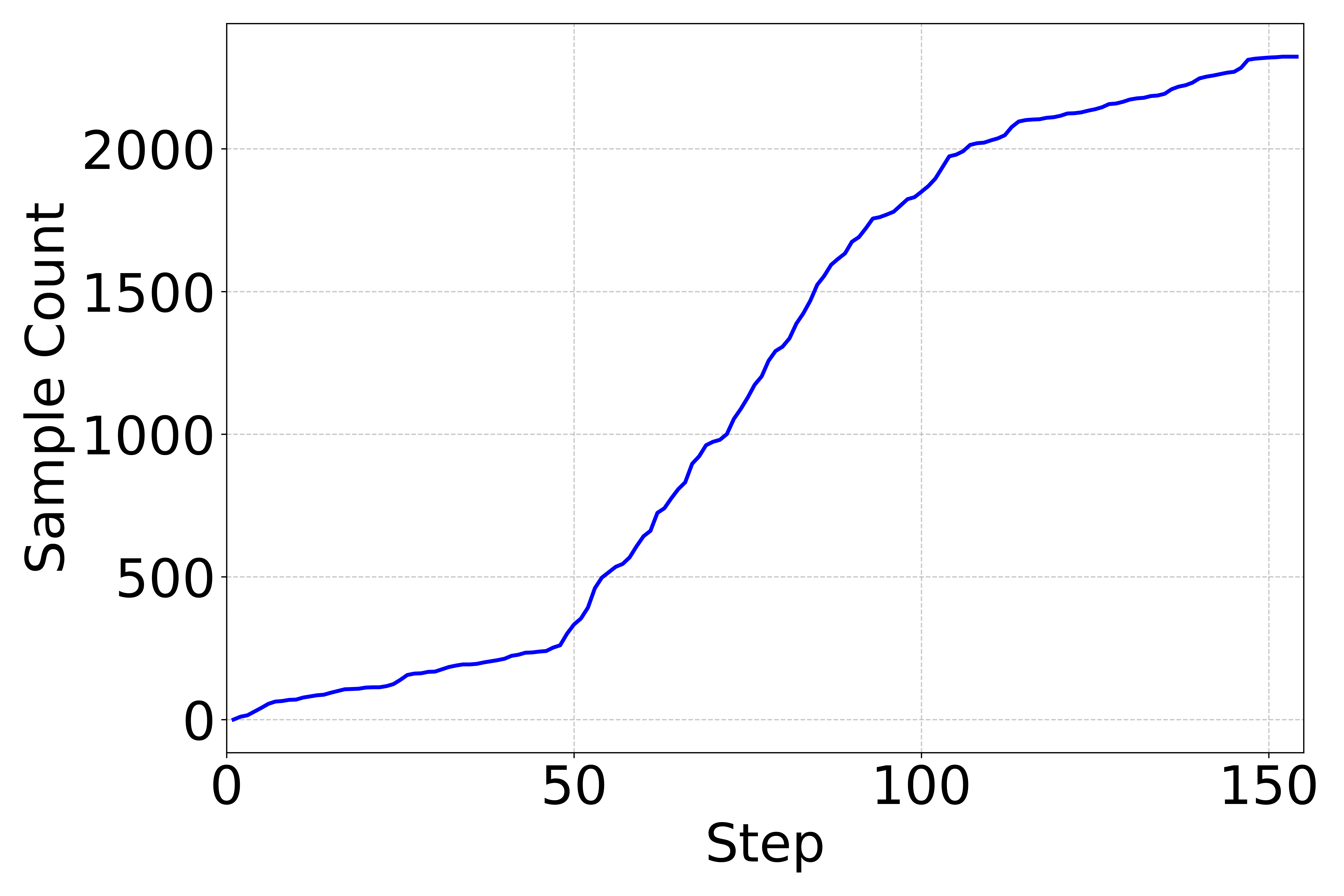}
        \caption{7 Characters}
    \end{subfigure}
    \begin{subfigure}[t]{0.33\columnwidth}
        \centering
        \includegraphics[width=\columnwidth]{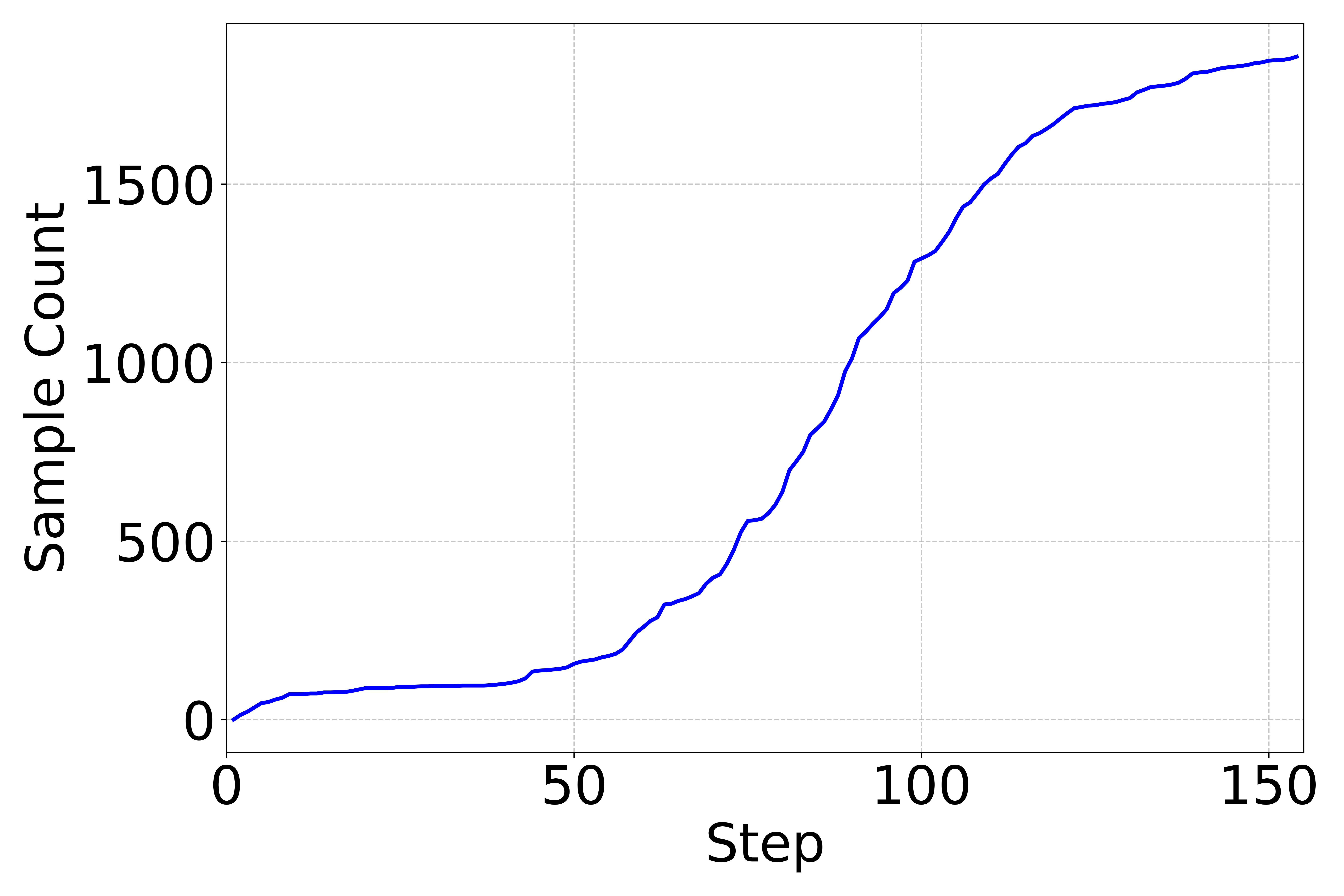}
        \caption{8 Characters}
    \end{subfigure}
    
    \begin{subfigure}[t]{0.33\columnwidth}
        \centering
        \includegraphics[width=\columnwidth]{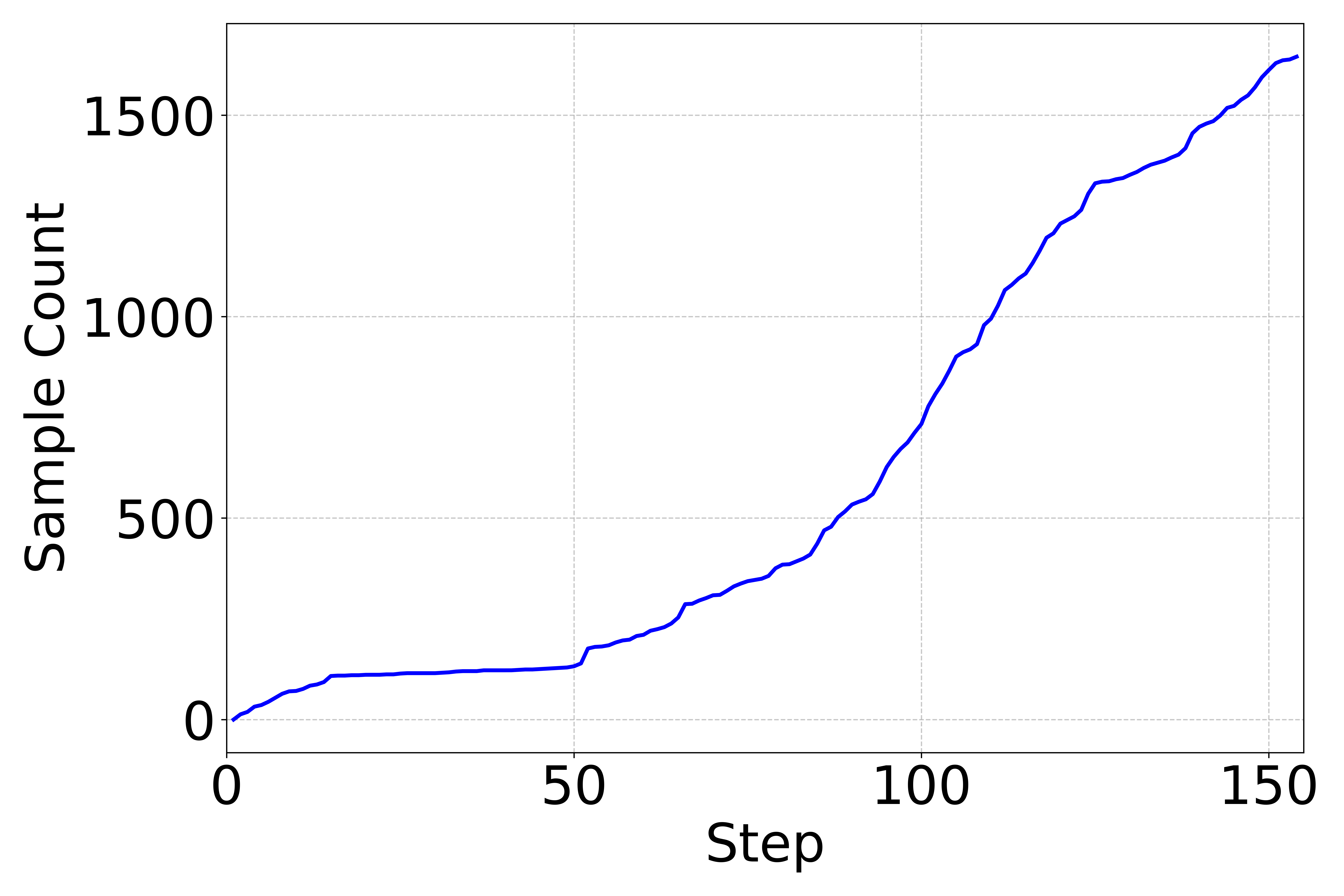}
        \caption{9 Characters}
    \end{subfigure}
    \begin{subfigure}[t]{0.33\columnwidth}
        \centering
        \includegraphics[width=\columnwidth]{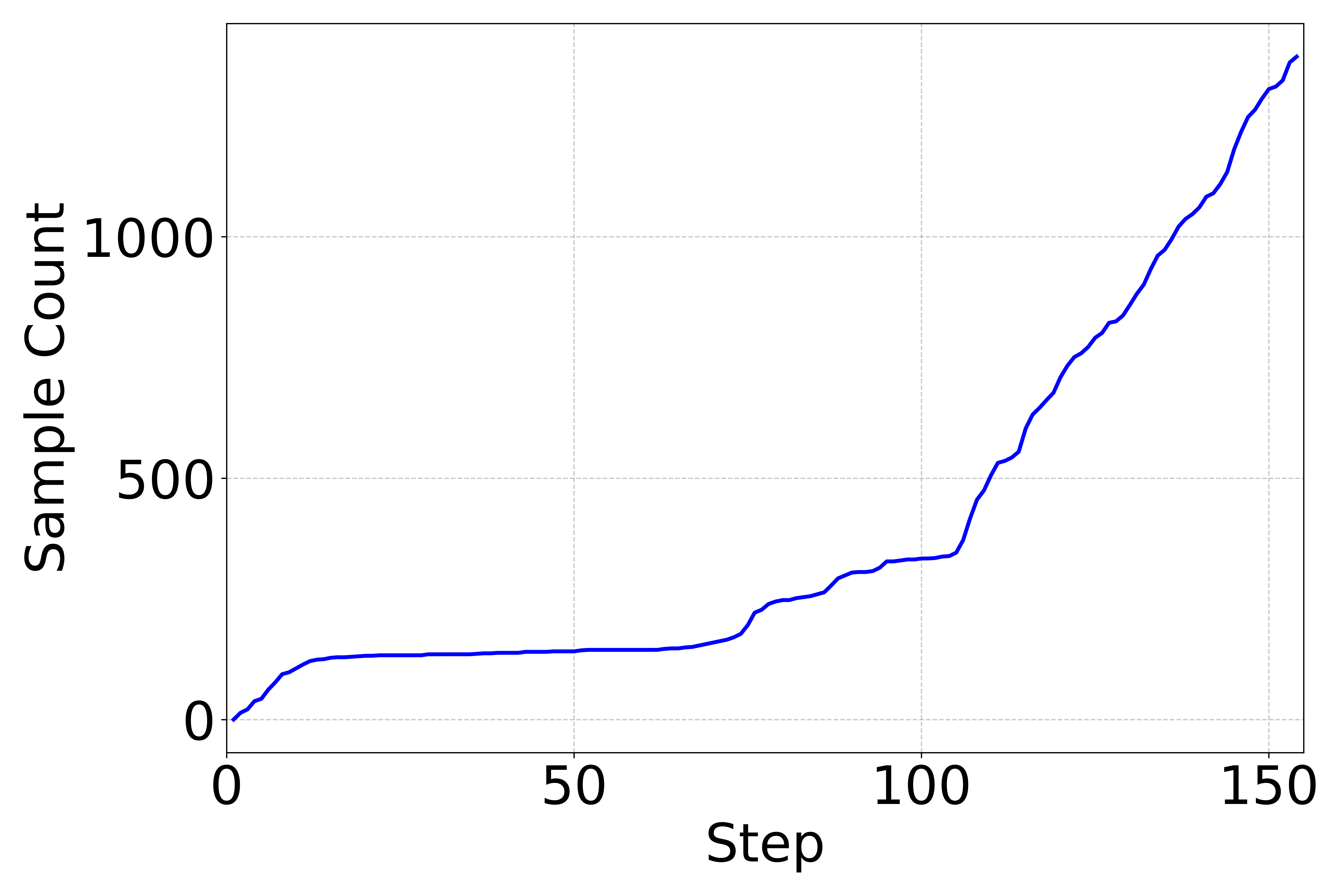}
        \caption{10 Characters}
    \end{subfigure}
    \begin{subfigure}[t]{0.33\columnwidth}
        \centering
        \includegraphics[width=\columnwidth]{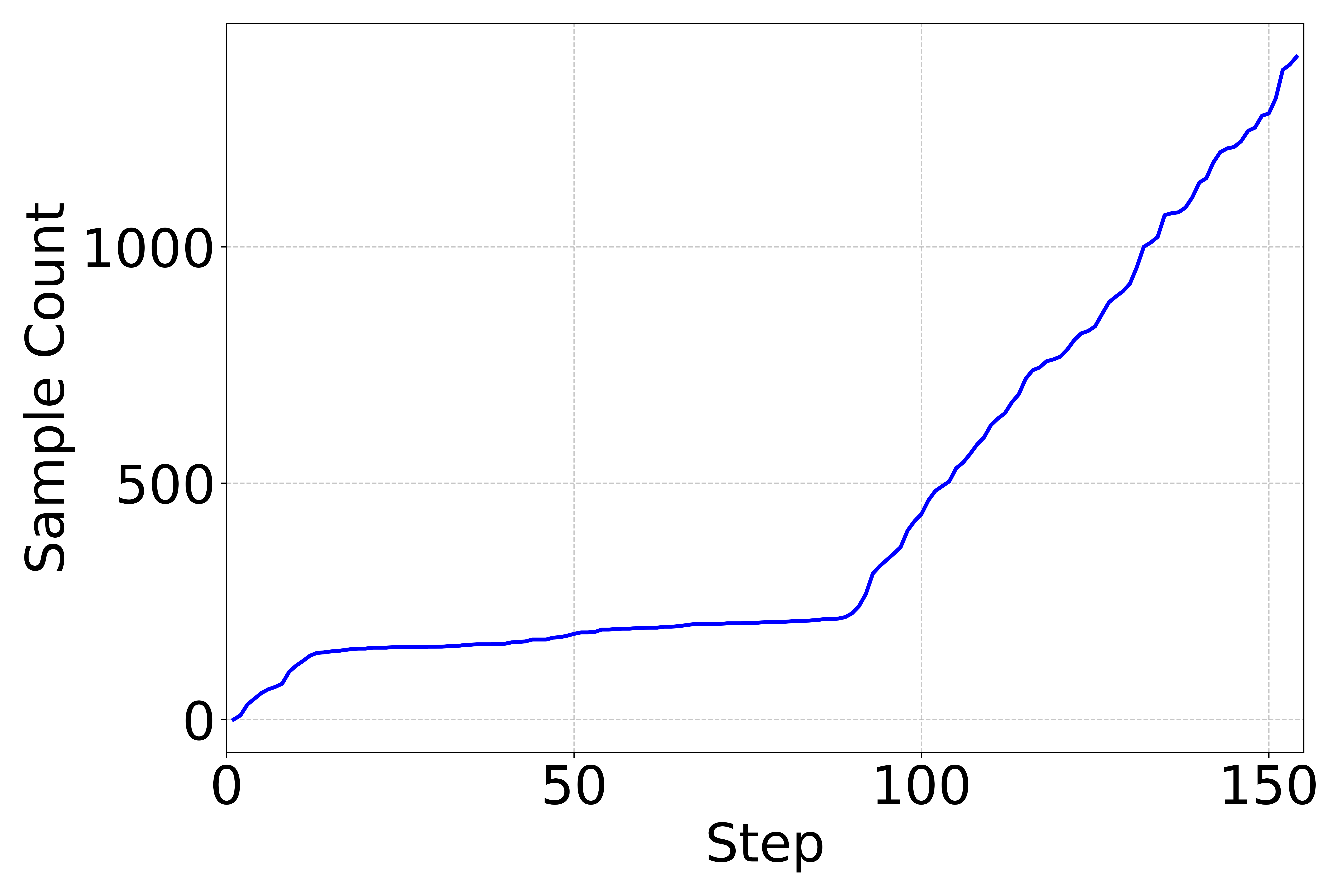}
        \caption{11 Characters}
    \end{subfigure}
    
    \begin{subfigure}[t]{0.33\columnwidth}
        \centering
        \includegraphics[width=\columnwidth]{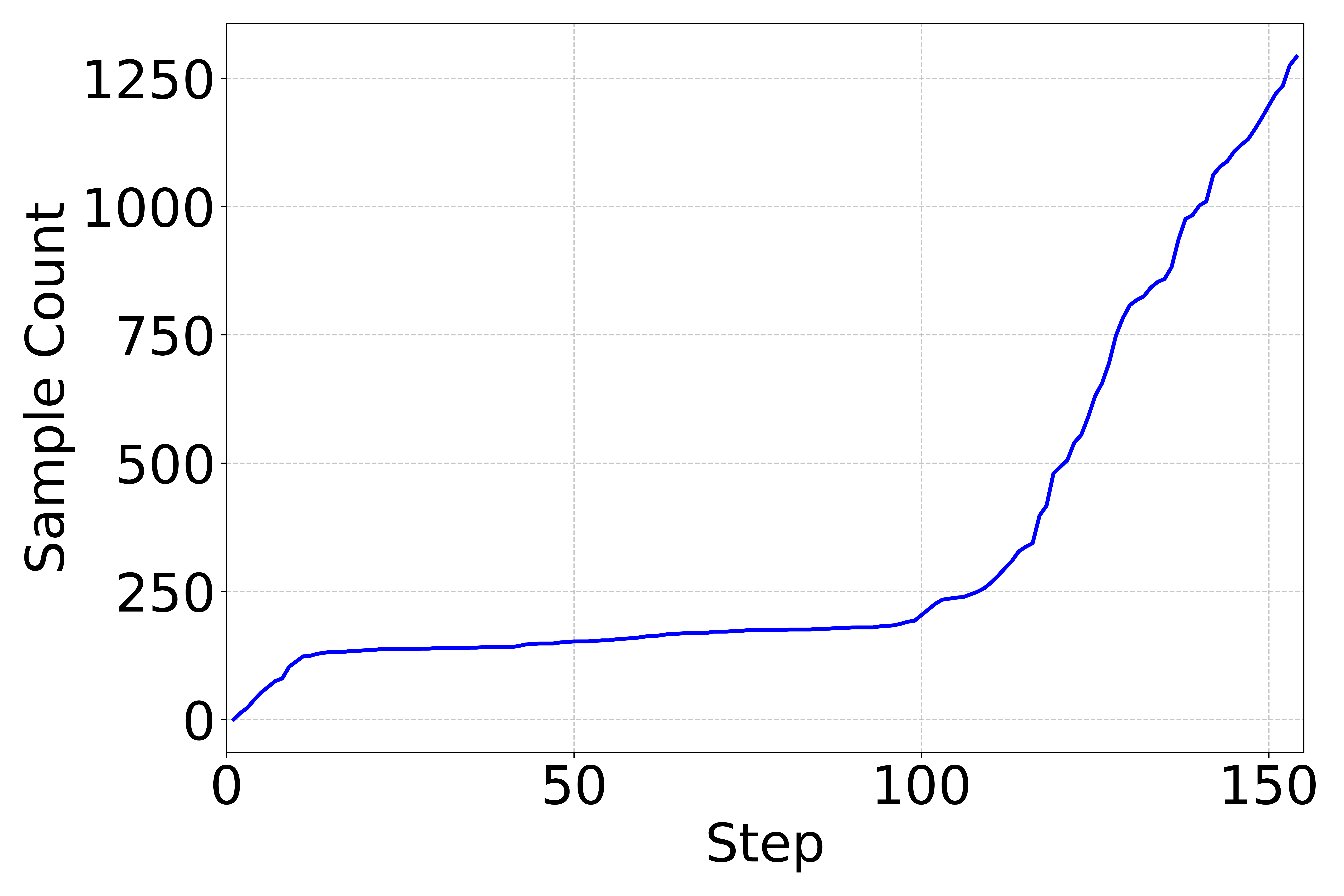}
        \caption{12 Characters}
    \end{subfigure}
    \begin{subfigure}[t]{0.33\columnwidth}
        \centering
        \includegraphics[width=\columnwidth]{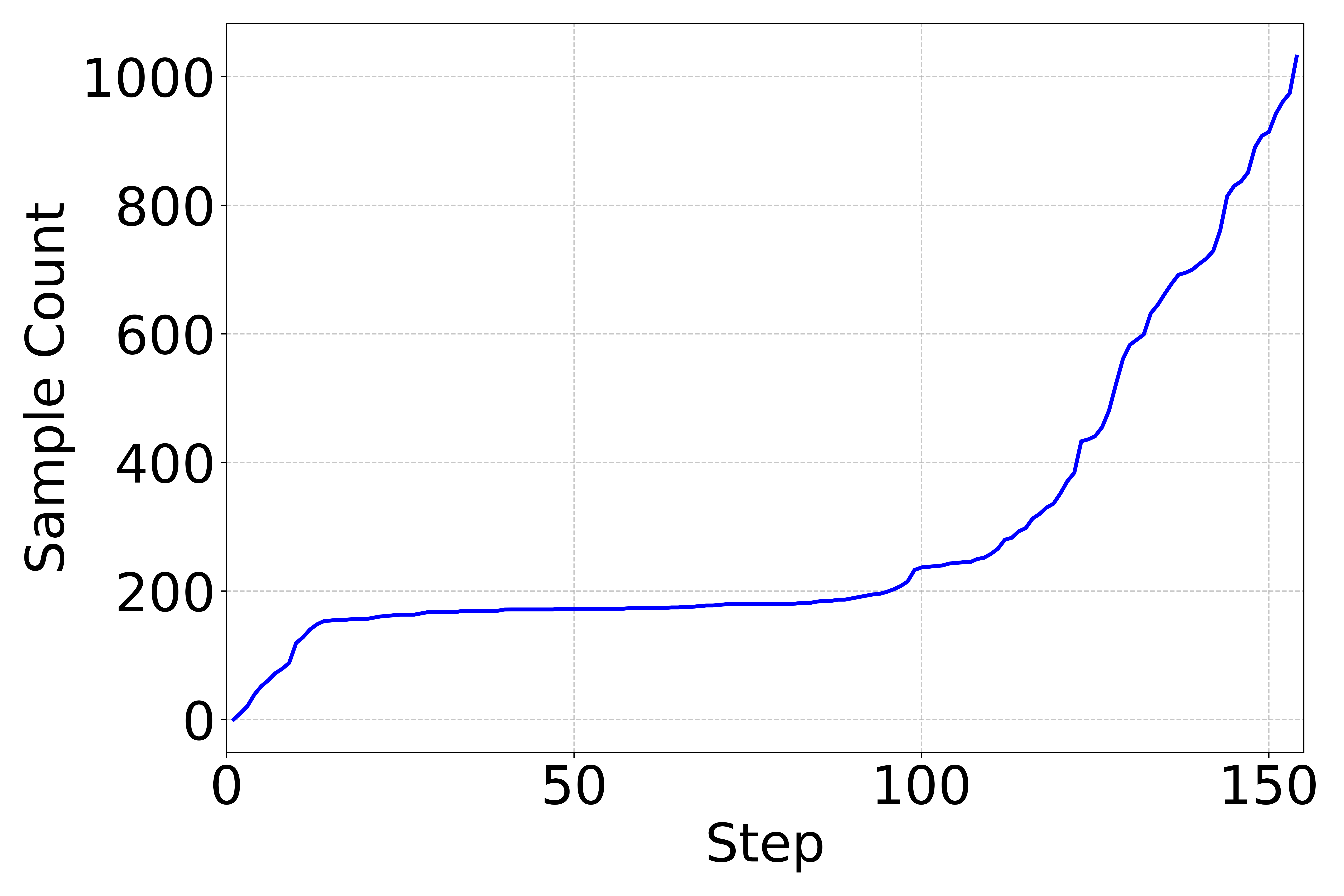}
        \caption{13 Characters}
    \end{subfigure}
    \begin{subfigure}[t]{0.33\columnwidth}
        \centering
        \includegraphics[width=\columnwidth]{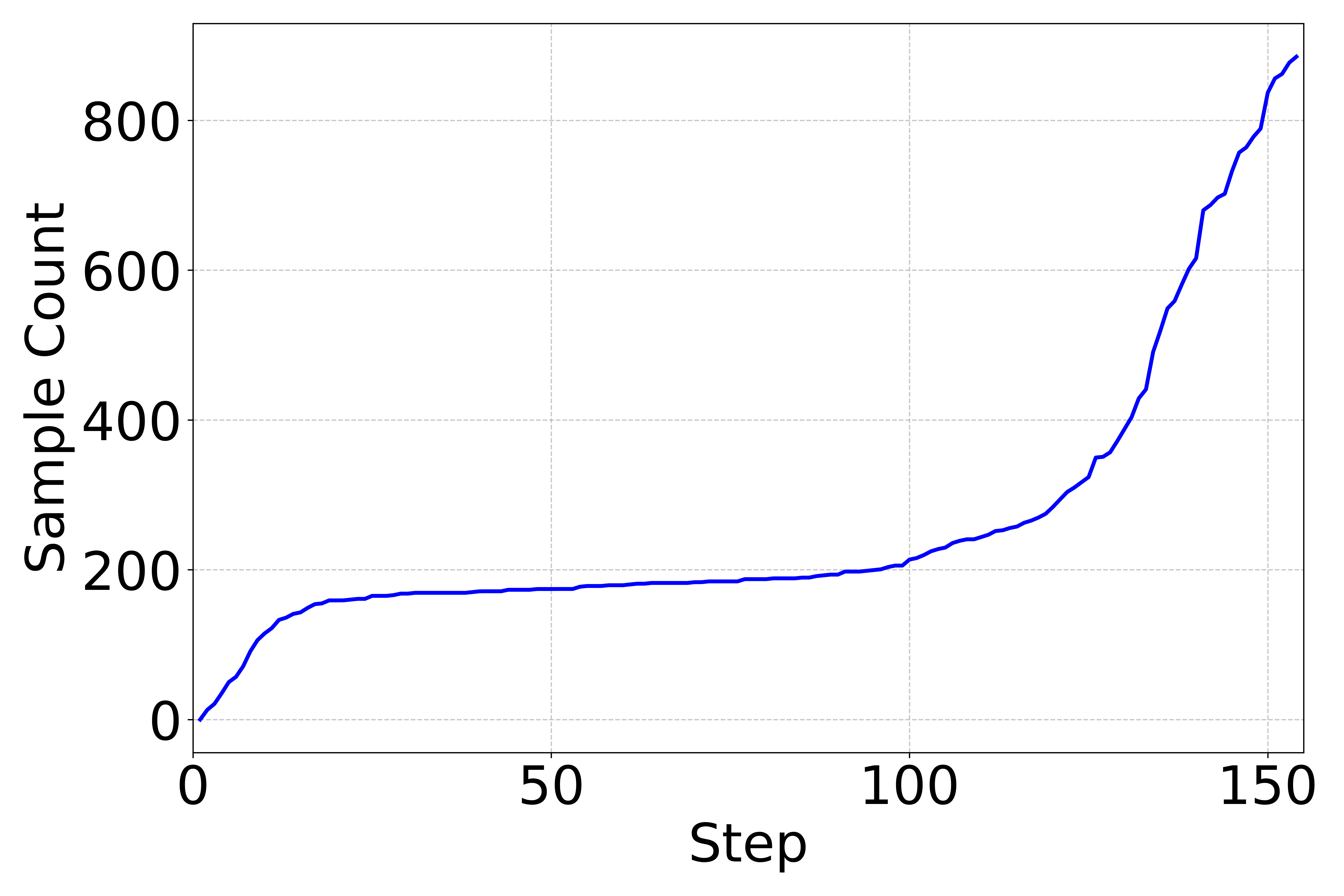}
        \caption{14 Characters}
    \end{subfigure}
    
    \caption{Curriculum (sample counts) induced by \sys{} across 12 K\&K puzzle distributions with increasing difficulty defined by the number of characters in each puzzle (Setting 1). Simpler distributions are automatically prioritized in early training, while more complex ones are progressively emphasized—both in an entirely automated manner—demonstrating automated distribution scheduling.}
    \label{fig:curriculum_profiles}
    \vspace{-0.4cm}
\end{figure}

\vspace{-0.2cm}
\subsection{Effectiveness of \sys}
\vspace{-0.2cm}

\emph{Setting 1: Post-training on the combination of K\&K puzzle datasets with different number of characters.} To evaluate the effectiveness of \sys{} in improving post-training efficiency and performance, we compare it against a uniform distribution sampling baseline across 12 distinct data distributions in the K\&K puzzle dataset. Each distribution corresponds to a fixed number of characters in the puzzle, ranging from 3 to 14. \autoref{fig:answer_reward} plots the test answer reward over training steps for each distribution, with and without \sys.
Across all distributions, \sys{} consistently outperforms the baseline, achieving faster convergence and higher test performance. The gains are particularly notable in mid- to high-difficulty distributions (e.g., 6 to 12 characters), where uniform sampling tends to struggle due to data underutilization. For example, in the 9-character distribution (\autoref{fig:answer_reward}g), the model trained with \sys{} achieves a reward of over 0.5, whereas the baseline remains below 0.0.
These results validate the core intuition of \sys: dynamically adjusting the sampling focus toward high-learnability distributions accelerates policy improvement while avoiding wasted effort on over-saturated or low-signal data. Notably, the improvement is achieved without any curriculum heuristics or manual data ordering—only by observing advantage signals and adapting online.

\emph{Setting 2: Post-training on diverse logic reasoning distributions.}
We apply \sys{} to 15 logic reasoning distributions including subsets of RuleTaker, ProofWriter, and K\&K (with varying difficulty levels), as well as datasets such as AR-LSAT, LogiQA, LogicNLI, and LongICLBench. As shown in \autoref{tab:combined_logic}, \sys{} improves the average test answer reward from 0.90 to 1.17. Notable improvements are observed on complex tasks such as AR-LSAT, where the reward increases from -0.70 to -0.52, and K\&K 7 Characters, from 0.56 to 1.02. These results demonstrate that \sys{} adaptively prioritizes undertrained but learnable distributions, leading to more efficient capability gains.

\emph{Setting 3: Post-training on diverse math data distributions.}
We further evaluate \sys{} on GSM-8K and different subsets of AIME grouped by competition years. As shown in \autoref{tab:math}, \sys{} raises the average test answer reward from -0.59 to -0.43, with the most significant gain on AIME 1994–2004, where performance improves from -1.50 to -1.02. These results highlight \sys{}’s robustness under distribution shifts and data imbalance.

\vspace{-0.2cm}
\subsection{Ablation Study on the Sampling Strategy}
\vspace{-0.2cm}

In this section, we ablate the sampling strategy used in \sys{}’s UCB-based scheduler. As described in \autoref{alg:cl}, our method applies soft sampling controlled by a temperature parameter. The greedy variant (temperature = 0) always selects the distribution with the highest UCB score, while our default uses a small temperature (0.1) to enable probabilistic sampling.
We conduct experiments under Setting 1, with a maximum training response length of 10240 tokens. After 100 training steps, the greedy strategy significantly underperforms due to its lack of exploration—it tends to lock onto a single distribution early and fails to adapt. For instance, on the 13- and 14-character K\&K tasks, the greedy variant achieves test answer rewards of $-0.91$ and $-1.38$, while soft sampling reaches $-0.66$ and $-1.16$, respectively.
These results highlight the importance of maintaining exploration via a non-zero temperature to prevent the scheduler from collapsing onto suboptimal distributions.

\vspace{-0.2cm}
\subsection{Analyzing the Automated Curriculum by \sys}
\vspace{-0.2cm}

To understand how \sys{} dynamically allocates training effort across data distributions, we analyze the sampling patterns induced by its UCB-based curriculum mechanism. \autoref{fig:curriculum_profiles} shows the cumulative number of samples drawn from each distribution (3 to 14 characters) over the course of training on K\&K puzzles with varying character numbers (Setting 1).
We observe a clear curriculum-like progression: distributions corresponding to simpler puzzles (e.g., 3–5 characters) are heavily sampled in the early stages of training, while more complex distributions (e.g., 10–14 characters) are gradually introduced and increasingly prioritized as training progresses. This pattern aligns with the model's evolving capacity—early training favors distributions with high initial advantage magnitudes, and as the model saturates on those, \sys{} shifts focus to underexplored but learnable distributions.
Importantly, this adaptive sampling behavior emerges automatically from empirical advantage signals without requiring manual specification of curriculum order. These results highlight \sys{}’s ability to construct an implicit, data-driven curriculum that mirrors traditional easy-to-hard strategies, while remaining responsive to online training dynamics.

\vspace{-0.2cm}
\section{Conclusion}
\vspace{-0.2cm}
\label{sec:conclusion}
In this work, we introduce a distribution-level curriculum learning framework for RL-based post-training of large language models. \sys{} leverages the expected absolute advantage as a learnability signal to adaptively allocate training focus across heterogeneous distributions. By formalizing scheduling as a multi-armed bandit and adopting a UCB-based sampling strategy, \sys{} balances exploitation and exploration in a principled way. Experiments demonstrate that \sys{} consistently improves convergence and final performance over baselines. These results highlight the value of distribution-aware curriculum learning in LLM RL post-training.

\bibliographystyle{unsrtnat}
\bibliography{reference}

\newpage

\appendix

\section{Proof for \autoref{th:learnability}}
\label{sec:proof_learnability}
\begin{theorem}[Expected Advantage Magnitude Reflects Learnability]
Given a policy \(\pi_{\theta}\) and a data distribution \(d\),
the expected absolute advantage \(\mathbb{E}_{x \sim d} \left[ \mathbb{E}_{o_i \sim \pi_\theta(\cdot|x)} \left[ |\hat{A}_i| \right] \right]\) serves as a proxy for how much that distribution \(d\) can help the model improve, where the distribution $d$ consisting of prompts $x \sim d$, each prompt has a group of sampled outputs $\{o_1, \dots, o_n\}$, and \(\hat{A}_i\) denotes the advantage of output \(o_i\).
\end{theorem}

\begin{proof}
Let \(\pi_\theta\) be the current model policy. Consider a data distribution \(d\), where \(x \sim d\) are prompts and \(\{o_1, \dots, o_n\} \sim \pi_\theta(\cdot|x)\) are sampled outputs. For each output \(o_i\), the advantage is estimated as
\[
\hat{A}_i = r_i - b(x),
\]
where \(r_i\) is the reward assigned to \(o_i\) and \(b(x)\) is a baseline (e.g., the mean reward over the group). The policy gradient under common policy-gradient methods (e.g., PPO or GRPO) can be written as:
\[
\nabla_\theta \mathcal{J}(\theta) = \mathbb{E}_{x \sim d} \left[ \mathbb{E}_{o_i \sim \pi_\theta(\cdot|x)} \left[ \hat{A}_i \cdot \nabla_\theta \log \pi_\theta(o_i \mid x) \right] \right].
\]

Now consider the magnitude of the gradient vector. The strength of the training signal from \(d\) depends on the expected norm of the gradient, which is bounded below by:
\[
\left\| \nabla_\theta \mathcal{J}(\theta) \right\| \gtrsim \mathbb{E}_{x \sim d} \left[ \mathbb{E}_{o_i \sim \pi_\theta(\cdot|x)} \left[ |\hat{A}_i| \cdot \left\| \nabla_\theta \log \pi_\theta(o_i \mid x) \right\| \right] \right].
\]

Assuming that \(\left\| \nabla_\theta \log \pi_\theta(o_i \mid x) \right\|\) is bounded and varies slowly across \(d\), the dominant term affecting the gradient norm is:
\[
\mathbb{E}_{x \sim d} \left[ \mathbb{E}_{o_i \sim \pi_\theta(\cdot|x)} \left[ |\hat{A}_i| \right] \right].
\]

Thus, the expected absolute advantage serves as a proxy for the learning signal magnitude contributed by distribution \(d\). The expected absolute advantage reflects how much training on distribution \(d\) can improve the model parameters, making it a suitable signal for curriculum scheduling.

\end{proof}

\section{Theoretical Justification for UCB-Based Distribution Scheduling}
\label{sec:theory_ucb}

We provide a theoretical justification for using Upper Confidence Bound (UCB) as a strategy for scheduling training over data distributions in RL-based post-training. Our objective is to maximize the cumulative learnability gain over \(T\) training steps, defined as:
\[
\max_{\{d_t\}_{t=1}^T} \sum_{t=1}^{T} L(d_t), \quad \text{where} \quad L(d) = \mathbb{E}_{x \sim d} \left[ \mathbb{E}_{o \sim \pi_\theta(\cdot|x)} \left[ |\hat{A}(o)| \right] \right].
\]

This setting can be viewed as a stochastic multi-armed bandit (MAB) problem, where each data distribution \(d_j \in \mathcal{D}\) corresponds to an arm with unknown reward \(L(d_j)\), interpreted as the expected absolute advantage from training on samples from \(d_j\). At each training step \(t\), the learner selects a distribution \(d_t\) and obtains an empirical reward \(\hat{L}(d_t)\) by averaging the absolute advantages observed in the batch.

We define the regret as the gap between the cumulative learnability gain of the best fixed distribution \(d^* = \arg\max_d L(d)\) and that of the learner's actual selections:
\[
\text{Regret}(T) = \sum_{t=1}^T L(d^*) - \sum_{t=1}^T L(d_t).
\]

To analyze this regret, we make the following assumptions:

\begin{enumerate}
    \item For each distribution \(d_j\), the per-output absolute advantages \(|\hat{A}(o)|\), where \(o \sim \pi_\theta(\cdot|x)\), are i.i.d. and bounded in \([0, C]\) for some constant \(C > 0\).
    \item The true expected advantage \(L(d_j)\) remains approximately stationary over a local training window, enabling meaningful online adaptation.
\end{enumerate}

\textit{Note:} In practice, we can clip or normalize \(|\hat{A}(o)|\) to satisfy the boundedness condition. The introduction of the constant \(C\) only scales the regret by a constant factor and does not affect the asymptotic rate of convergence.

Under these assumptions, the following regret bound holds:

\begin{theorem}
\label{thm:ucb_regret}
Let \(\mathcal{D} = \{d_1, \dots, d_N\}\) be a set of data distributions with fixed expected rewards \(L(d_j) \in [0, C]\). Then, applying the UCB1 algorithm to the empirical reward observations yields the regret bound:
\[
\text{Regret}(T) \leq O\left( C \cdot \sum_{j: \Delta_j > 0} \frac{\log T}{\Delta_j} \right),
\quad \text{where} \quad \Delta_j = L(d^*) - L(d_j).
\]
\end{theorem}

\begin{proof}
This result is a direct application of the classical UCB1 regret bound~\citep{auer2002finite}, extended to the case where reward values lie in \([0, C]\). Let \(d^* = \arg\max_d L(d)\) be the optimal distribution, and let \(\Delta_j = L(d^*) - L(d_j)\) denote the suboptimality gap for each arm \(d_j\).

At each time step \(t\), UCB1 selects the distribution \(d_j\) with the highest upper confidence bound:
\[
\text{UCB}(d_j) = \hat{L}(d_j) + \sqrt{ \frac{2C^2 \log t}{n_j} },
\]
where \(n_j\) is the number of times distribution \(d_j\) has been sampled so far, and \(\hat{L}(d_j)\) is the empirical mean of observed rewards (mean absolute advantages).

Under the assumptions that rewards are i.i.d. and bounded in \([0, C]\), the Hoeffding inequality guarantees that with high probability the empirical mean concentrates around the true mean \(L(d_j)\), and the UCB selection mechanism will only pick suboptimal arms a logarithmic number of times. Based on UCB1 regret bound~\citep{auer2002finite}, The cumulative regret is therefore bounded by:
\[
\text{Regret}(T) \leq \sum_{j: \Delta_j > 0} \left( \frac{8 C^2 \log T}{\Delta_j} + O(\Delta_j) \right),
\]
which simplifies to the stated asymptotic bound:
\[
\text{Regret}(T) = O\left( C \cdot \sum_{j: \Delta_j > 0} \frac{\log T}{\Delta_j} \right).
\]
\end{proof}

This result shows that our distribution-level scheduling strategy, when driven by UCB over empirical advantage rewards, is provably efficient. It dynamically concentrates training on distributions with high estimated learnability while ensuring sufficient exploration, with regret that scales logarithmically in \(T\) and linearly in \(1/\Delta_j\).

\section{Comparison to Heuristic Curriculum}
\label{sec:comparison}
Heuristic curricula, which manually specify a fixed training schedule over data distributions—e.g., training on Distribution A for N steps before switching to Distribution B—have been explored in prior work~\cite{xie2025logic,team2025kimi}, particularly in environments where task difficulty or domain progression is well understood. However, such approaches have several limitations that make them less suitable for our setting.
First, effective heuristic scheduling requires strong prior knowledge about the relative difficulty and learnability of each distribution. In our setting, which involves diverse domains such as logic reasoning, mathematics, and programming, such prior knowledge is often unavailable or misleading. For example, a distribution may appear “easier” but provide low learning signal, or seem “harder” but actually yield high gradient utility. This makes it extremely difficult to construct robust, generalizable heuristics across tasks.
Second, heuristic curricula are static and cannot adapt to the evolving needs of the model during training. In contrast, DUMP dynamically adjusts sampling priorities based on actual model performance—measured via policy advantages—allowing it to focus on the most beneficial distributions at each stage of learning.
Finally, the lack of standardized or widely accepted heuristic curricula for our task suite makes it hard to conduct fair and meaningful comparisons. Instead, we benchmark DUMP against uniform sampling and adaptive baselines, which are more reflective of current best practices in large-scale post-training pipelines.

\section{Limitations}
\label{sec:limitation}
First, while the core idea of distribution-level curriculum learning is broadly applicable, we evaluate DUMP only in the context of large language models (LLMs) and do not extend the experiments to multimodal large language models (MLLMs) due to computational constraints. Second, our experiments are limited to 7B-scale models. Scaling our method to larger models remains an important direction for future work.

\begin{figure}[t]
		\centering
		\footnotesize
    \includegraphics[width=1\columnwidth]{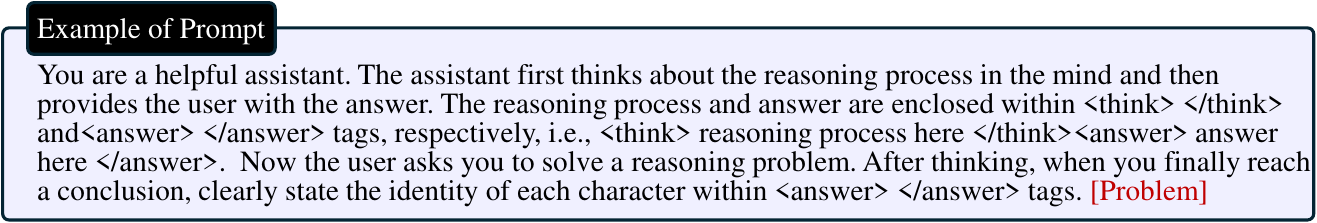}
	\caption{Example of prompt used.}\label{fig:prompt_example}
\end{figure}

\end{document}